\documentclass[twoside,leqno,twocolumn]{article}

\usepackage[letterpaper]{geometry}

\usepackage{ltexpprt}

\usepackage{hyperref}       
\usepackage{url}            
\usepackage{booktabs}       
\usepackage{amsfonts}       
\usepackage{nicefrac}       
\usepackage{microtype}      
\usepackage{amsmath}
\usepackage{graphicx}
\usepackage{subcaption}
\usepackage{xcolor}
\usepackage{lastpage}
\usepackage{adjustbox}
\usepackage{cite}

\newcommand{\cV}{\mathcal{V}}
\newcommand{\cE}{\mathcal{E}}
\newcommand{\bA}{\mathbf{A}}
\newcommand{\cS}{\mathcal{S}}
\newcommand{\cG}{\mathcal{G}}
\newcommand{\cZ}{\mathcal{Z}}
\newcommand{\cO}{\mathcal{O}}
\newcommand{\cD}{\mathcal{D}}
\newcommand{\bz}{\mathbf{z}}
\newcommand{\cQ}{\mathcal{Q}}
\newcommand{\sm}{\text{SIM}}
\newcommand{\dsm}{\text{DISSIM}}
\newcommand{\coh}{\text{COHESION}}
\newcommand{\sep}{\text{SEPARATION}}

\newcommand*\samethanks[1][\value{footnote}]{\footnotemark[#1]}

\begin{document}

\title{\Large Maximizing Cohesion and Separation in Graph Representation Learning: \\ A Distance-aware Negative Sampling Approach}
\author{M. Maruf\thanks{Dept. of Computer Science, Virginia Tech, VA, USA}
\and Anuj Karpatne\samethanks}

\date{}

\maketitle

\fancyfoot[R]{\scriptsize{Copyright \textcopyright\ 2021 by SIAM\\
Unauthorized reproduction of this article is prohibited}}


\begin{abstract} \small\baselineskip=9pt 
The objective of unsupervised graph representation learning (GRL) is to learn a low-dimensional space of node embeddings that reflect the structure of a given unlabeled graph. Existing algorithms for this task rely on negative sampling objectives that maximize the similarity in node embeddings at nearby nodes (referred to as ``cohesion'') by maintaining positive and negative corpus of node pairs. While positive samples are drawn from node pairs that co-occur in short random walks, conventional approaches construct negative corpus by uniformly sampling random pairs, thus ignoring valuable information about structural dissimilarity among distant node pairs (referred to as ``separation''). In this paper, we present a novel Distance-aware Negative Sampling (DNS) which maximizes the separation of distant node-pairs while maximizing cohesion at nearby node-pairs by setting the negative sampling probability proportional to the pair-wise shortest distances. Our approach can be used in conjunction with any GRL algorithm and we demonstrate  the efficacy of our approach over baseline negative sampling methods over downstream node classification tasks on a number of benchmark datasets and GRL algorithms. All our codes and datasets are available at \url{https://github.com/Distance-awareNS/DNS/}.
\end{abstract}


\section{Introduction}


The goal of graph representation learning (GRL) is to learn a low-dimensional embedding of every node in the graph that captures the structure of interactions among nodes. The learned embeddings can be used as input features in downstream tasks such as network classification or link prediction. In GRL problems where every node has an associated set of attributes and target labels, e.g., over many benchmark datasets such as CiteSeer, Cora, and PubMed \cite{DBLP:journals/corr/YangCS16}, one can employ supervised learning methods to extract node embeddings \cite{DBLP:journals/corr/KipfW16, DBLP:journals/corr/GilmerSRVD17, DBLP:journals/corr/abs-1710-03059, hamilton2017inductive, velivckovic2017graph} that achieve state-of-the-art performance.
However, in a general GRL problem, we may not always have access to node features or labels, or the node features may be available in complex and varying formats (e.g., as molecular structures in protein-protein interaction or drug-drug interaction graphs). Further, we may be interested in learning a ``universal'' embedding of the nodes that captures the graph structure and is independent of downstream supervised learning tasks. Such a universal representation can then be used as input features for a new downstream task without re-training the embeddings. For these reasons, we focus our attention to the problem of unsupervised GRL, where the node embeddings are required to be learned solely from the graph structure (i.e., the adjacency matrix) and we do not consider the presence of any node or edge attributes or labels. Henceforth, we will use the term GRL to  refer to unsupervised GRL.

Most GRL algorithms are rooted in the idea of distributional similarity developed in the natural language processing (NLP) community \cite{mikolov2013efficient}, whereby words appearing in similar \emph{contexts} (e.g., sentences in a document) are mapped to similar representations. Similarly, most GRL algorithms aim to maximize the similarity of embeddings at nearby nodes, which are assumed to belong to similar contexts based on the structure of the graph. This is generally performed by maintaining a \emph{positive} corpus of nearby node-pairs (termed {positive} pairs) and a \emph{negative} corpus of randomly sampled node-pairs (termed {negative} pairs). The similarity of embeddings over the positive corpus is then contrasted with that over the negative corpus, and their difference is maximized to ensure positive pairs occupy similar embeddinggs. A common strategy for sampling the negative pairs is to use a unigram distribution over all nodes, referred to as the unigram negative sampling (UNS) method.

While maximizing the similarity at nearby node-pairs is an important objective, a second objective that is important yet mostly overlooked in existing GRL algorithms is to maximize the \emph{dissimilarity at distant node-pairs}. 
This is important because ideally, we would like to learn embeddings where the structural similarity of nodes (e.g., based on the distance of the shortest path between two nodes, or network distance) is preserved in the embedding space. In other words, the similarity of node-pairs in the embedding space should be proportional to their network distance. As a result, by maximizing this second objective, we can obtain well-separated and meaningful embeddings, whereby node-pairs that are nearby occupy similar embeddings while those that are far apart occupy dissimilar embeddings. Using an analogy from the domain of clustering, we refer to the first objective as maximizing \emph{graph cohesion}, i.e., similarity at nearby nodes, and the second objective as maximizing \emph{graph separation}, i.e., dissimilarity at distant nodes. We present an intuitive negative sampler for maximizing both cohesion and separation in GRL by sampling negative pairs with probability proportional to the distance between the nodes, termed  as Distance-aware Negative Sampler (DNS).

\paragraph{Our Contributions:} (1) We introduce and define the concepts of cohesion and separation in the context of GRL. (2) We propose a novel Distance-aware Negative Sampler (DNS) that maximizes both cohesion and separation. (3)  We theoretically show the effectiveness of our DNS approach in maximizing cohesion and separation as compared to UNS. (4) We present a scalable DNS approach with reduced space and time complexity for large networks. (5) We empirically show the ability of our DNS approach to learn meaningful representations, thus leading to better predictive performance on downstream ML tasks on several benchmark datasets in comparison with baseline GRL algorithms. 

\section{Related work}
\label{sec:related_work}
\paragraph{Unsupervised graph representation learning methods:}
\label{subsection:baselines}

A number of existing unsupervised GRL methods maximize embedding similarity at nearby nodes directly without performing negative sampling. Some examples include matrix factorization based methods \cite{ahmed2013distributed, cao2015grarep, ou2016asymmetric} and skip-gram based methods \cite{perozzi2014deepwalk, DBLP:journals/corr/abs-1710-09599, armandpour2019robust}. Some GRL methods use a variety of negative sampling strategies to learn node embeddings. This category includes methods that use input node features such as Graph Convolutional Network (GCN) encoders \cite{kipf2016variational, hamilton2017inductive, velivckovic2018deep} that have achieved state-of-the-art performances on benchmark GRL datasets. However, they are not directly relevant to our GRL problem since we consider the formulation where no node features are available. Negative sampling based methods that do not use node features include node2vec \cite{grover2016node2vec}, which optimizes random walk objectives and LINE \cite{tang2015line}, which uses first- or second-order neighborhoods to construct similar nodes. Note that while DeepWalk \cite{perozzi2014deepwalk} was originally proposed using a Hierarchical Softmax objective,  we can adapt it to construct a negative sampling based version of DeepWalk.

\paragraph{Negative sampling strategies:}
\label{subsection:Negative_Sampling}
Here we discuss some of the common strategies for negative sampling that are at the basis of several unsupervised GRL algorithms. There are two generic types of negative samplers, edge-based \cite{kipf2016variational, tang2015line} and node-based \cite{grover2016node2vec, hamilton2017inductive}.  Edge-based samplers construct the positive corpus by selecting node pairs that have an edge between them, and the negative corpus by randomly selecting node pairs that do not have an edge. On the other hand, node-based samplers use random walk objectives to construct the positive corpus and select random node pairs distributed with unigram distribution to construct the negative corpus. Among unigram distributions, two are common; one chooses negative samples with uniform probability \cite{kipf2016variational, tang2015line} and the other uses degree-based probability \cite{grover2016node2vec, hamilton2017inductive}, where the negative sampling probability is proportional to the $\frac{3}{4}$th power of the degree of each node. It is known that degree-based unigram sampler suffers from the \textit{popular neighbor} problem \cite{armandpour2019robust}, as this approach may choose a nearby node with high degree as a negative sample. Henceforth, by Unigram Negative Sampler (UNS), we refer the unigram sampler with uniform probability, and unigram-deg/UNS-deg denotes degree-based unigram negative sampler. There are some more negative samplers that have been proposed in recent works \cite{armandpour2019robust, yang2020understanding}; however, none of them use the notion of network distances in negative sampling.
\section{Preliminaries and problem objective}
\label{sec:problem_definition}
\subsection{Notations}
We are given an undirected graph $\cG = (\cV, \cE)$ where $|\cV| = n$, $|\cE| = m$, and the adjacency matrix is given by $\bA = [a{(i,j)}]_{n\times n}$. We assume that the graph is unweighted such that $a({i,j}) = 1$ iff $ (i,j) \in \cE$, otherwise $0$. We denote the set of all possible node-pairs as $\cS = \cV \times \cV$. Further, for every node-pair $(i,j) \in \cS$, we denote the distance or length of the shortest path between the nodes as $d(i,j)$. Incidentally, $d(i,j) = 1$ iff $(i,j) \in \cE$, i.e., there exists an edge between nodes $i$ and $j$. Let us refer to the maximum value of $d(i,j)$ in graph $\cG$ as $d_{max}$. We can then talk about the subset of node-pairs whose distance is equal to $d$, i.e., $\cS_d = \{(i,j) \in \cS | d(i,j) = d\}$. It is easy to verify that $\cS = \cS_0 \cup \cS_1 \cup \ldots \cS_{d_{max}}$ and $\cS_1 = \cE$.

With this setup, we consider the problem of unsupervised GRL where the goal is to map every node $i$ to an $l$-dimensional vector embedding, $\bz_i \in \mathbb{R}^l$, such that the embedding space $\cZ = \{\bz_i\}_{i=1}^n$ preserves the structural properties of nodes in graph $\cG$ (typically, $l\ll|\cV|$). In particular, we consider two generic types of measures in the embedding space of a pair of nodes, (i) $\sm_\cZ (i,j):=$ similarity score between embeddings $\bz_i$ and $\bz_j$ (some examples include the dot product $\bz_i^T\bz_j$ and its monotonic transformations $\sigma(\bz_i^T\bz_j)$ and $\log(\sigma(\bz_i^T\bz_j))$, where $\sigma$ denotes the sigmoid function), and (ii) $\dsm_\cZ (i,j):=$ dissimilarity score between embeddings $\bz_i$ and $\bz_j$ (some examples include $-\bz_{ij}$, $\sigma(-\bz_i^T\bz_j)$ and $\log(\sigma(-\bz_i^T\bz_j))$). Note that there are multiple choices of similarity and dissimilarity functions to instantiate these two generic measures in any problem. Also, maximizing the similarity score of a node-pair is usually equivalent to minimizing its dissimilarity score for common function choices.

Ideally, we want to learn an embedding space $\cZ$ such that $\sm_\cZ (i,j)$ is large for nearby node-pairs (i.e., when $d(i,j)$ is small) and $\dsm_\cZ (i,j)$ is large for distant node-pairs (i.e., when $d(i,j)$ is large). This objective, which is at the basis of the distributional hypothesis in linguistics \cite{harris1954distributional}, can be expressed using the notions of \emph{cohesion} and \emph{separation} in GRL, formally defined in the following.

\subsection{Cohesion and Separation}
\vspace{-1ex}
\begin{Definition}
\textbf{Cohesion:} The cohesion of an embedding space $\cZ$ represents the aggregate similarity score between embeddings at nearby node-pairs in the graph. Formally, we define cohesion using the following weighted sum over similarity scores:
\vspace{-2ex}
\begin{align}
    \coh (\alpha, \cZ) = \sum_{d=1}^{d_{max}} \alpha_d \sum_{(i,j) \in \cS_d} \sm(i,j),\quad  \nonumber \\
    \text{where} ~ \alpha_d ~\geq ~ \alpha_{d+1}, ~~ \alpha_d \geq 0 ~~ \forall d. \nonumber
    \vspace{-1ex}
\end{align}
Observe that since the weights $\alpha_d$ monotonically decrease with $d$, this weighted sum pays greater emphasis to the similarity at nearby node-pairs (i.e., $\cS_d$ with small $d$). This is a generic definition of cohesion that can be instantiated using different choices of the weights $\alpha$. For example, if we specify $\alpha_1 = 1$ and $\alpha_d = 0 ~\forall~ d > 1$, then cohesion will be equal to the aggregate similarity over all the edges in $\cG$. As we will see later, a common approach for specifying $\alpha_d$ in most GRL algorithms is performing random walks and computing the probability of sampling a node-pair at a distance $d$ in the random walk.
\label{def:cohesion}
\end{Definition}

\begin{Definition}
\textbf{Separation:} The separation of an embedding space $\cZ$ captures the aggregate dissimilarity between embeddings at distant node-pairs. Similar to cohesion, we can formally define separation using the following weighted sum:
\vspace{-2ex}
\begin{align}
    \sep (\beta, \cZ) = \sum_{d=1}^{d_{max}} \beta_d \sum_{(i,j) \in \cS_d} \dsm(i,j),\quad  \nonumber \\
    \text{where} ~ \beta_d ~\leq ~ \beta_{d+1} ~~ \beta_d \geq 0 ~~ \forall d. \nonumber
\end{align}
In this generic definition, since $\beta_d$ monotonically increases with $d$, the dissimilarity at distant node-pairs have a greater contribution in the separation. Again, there can be multiple ways to instantiate $\beta_d$. For example, we can set $\beta_{d_{max}} = 1$ and $\beta_d = 0 ~\forall ~d < d_{max}$ such that the separation is equal to the dissimilarity at the farthest node-pairs in $\cG$.
\label{def:separation}
\end{Definition}

Since dissimilarity score is inversely related to similarity, it may seem that maximizing cohesion automatically maximizes separation. However, this is not true since the weighted sums involved in cohesion and separation focus on different subsets of node-pairs in $\cS$: while cohesion focuses on $\cS_d$ for small $d$, separation focuses on node-pairs in $\cS_d$ for large $d$.
We present the following theorem to prove this point. 

\vspace{-1ex}
\begin{theorem}
Given two embedding spaces, $\cZ^1$ and $\cZ^2$, $\coh(\alpha, \cZ^1) = \coh(\alpha, \cZ^2)$ does not imply that $\sep(\beta, \cZ^1) = \sep(\beta, \cZ^2)$, for all choices of $\alpha$ and $\beta$. 
\label{thm:1}
\end{theorem}
\vspace{-5ex}
\begin{proof}
We use a counter-example to show that there can exist multiple embedding spaces such that their cohesion values are equal but their separation values are different. Figure \ref{fig:proof_1_g} shows a toy graph with 4 nodes that is represented in three different two-dimensional embedding spaces in Figures \ref{fig:proof_1_a}, \ref{fig:proof_1_b}, and \ref{fig:proof_1_c}. If we specify cohesion to be the aggregate similarity over edges (shown as dotted lines), we can see that all three embedding spaces have the same cohesion. However, if we define separation to be the aggregate dissimilarity at farthest nodes (at distance 3), we can see that the separation ranges from large (Figure \ref{fig:proof_1_a}) to small (Figure \ref{fig:proof_1_c}).
\end{proof}

\begin{figure*}
\begin{subfigure}{.24\textwidth}
  \centering
  \includegraphics[width=.8\linewidth]{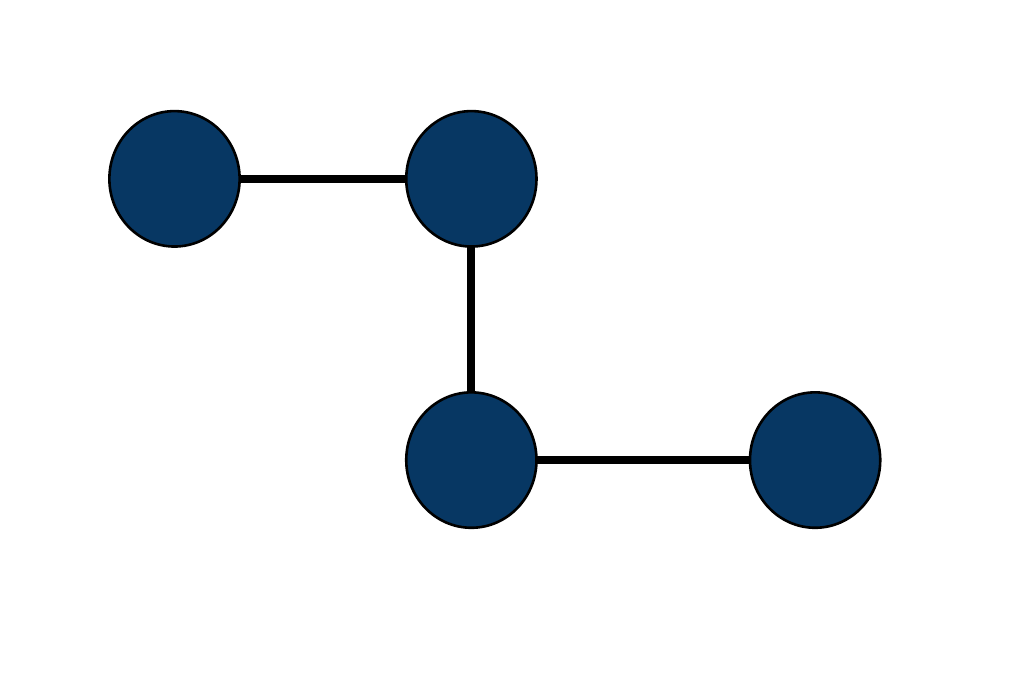} 
  \caption{Graph $\mathcal{G}$}
  \label{fig:proof_1_g}
\end{subfigure}%
\begin{subfigure}{.25\textwidth}
  \centering
  \includegraphics[width=.8\linewidth]{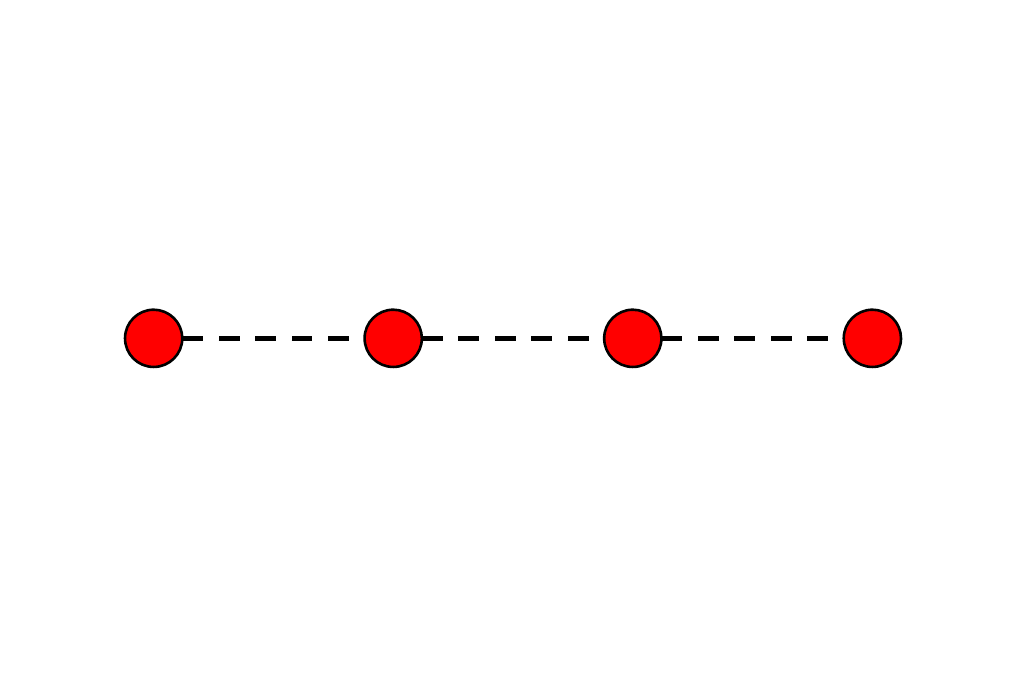}
  \caption{Large Separation}
  \label{fig:proof_1_a}
\end{subfigure}
\begin{subfigure}{.25\textwidth}
  \centering
  \includegraphics[width=.8\linewidth]{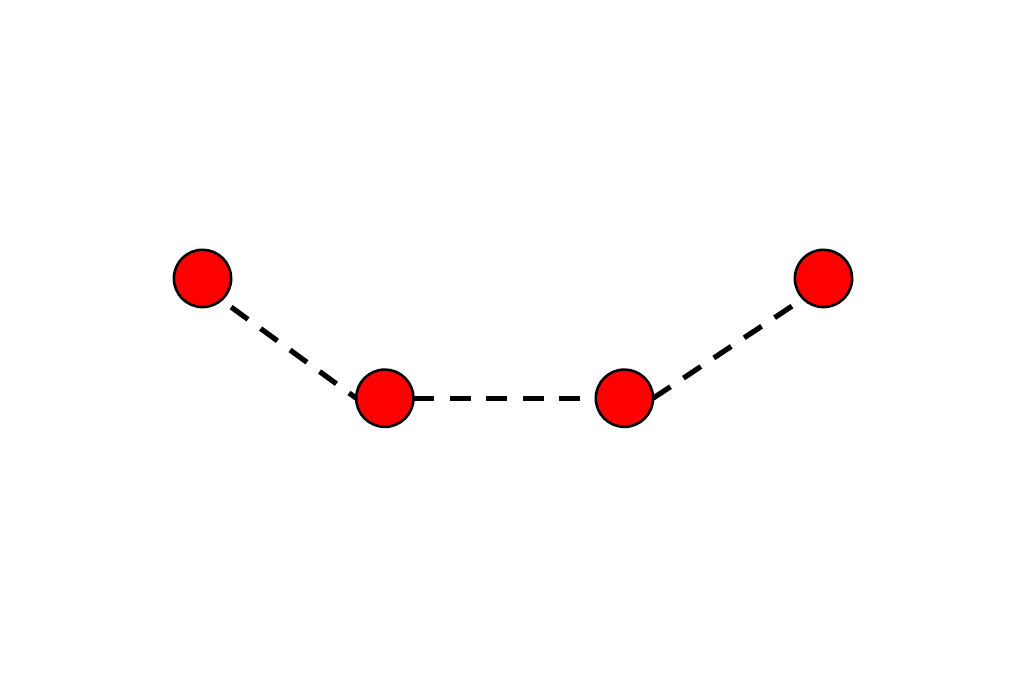}
  \caption{Moderate Separation}
  \label{fig:proof_1_b}
\end{subfigure}
\begin{subfigure}{.25\textwidth}
  \centering
  \includegraphics[width=.8\linewidth]{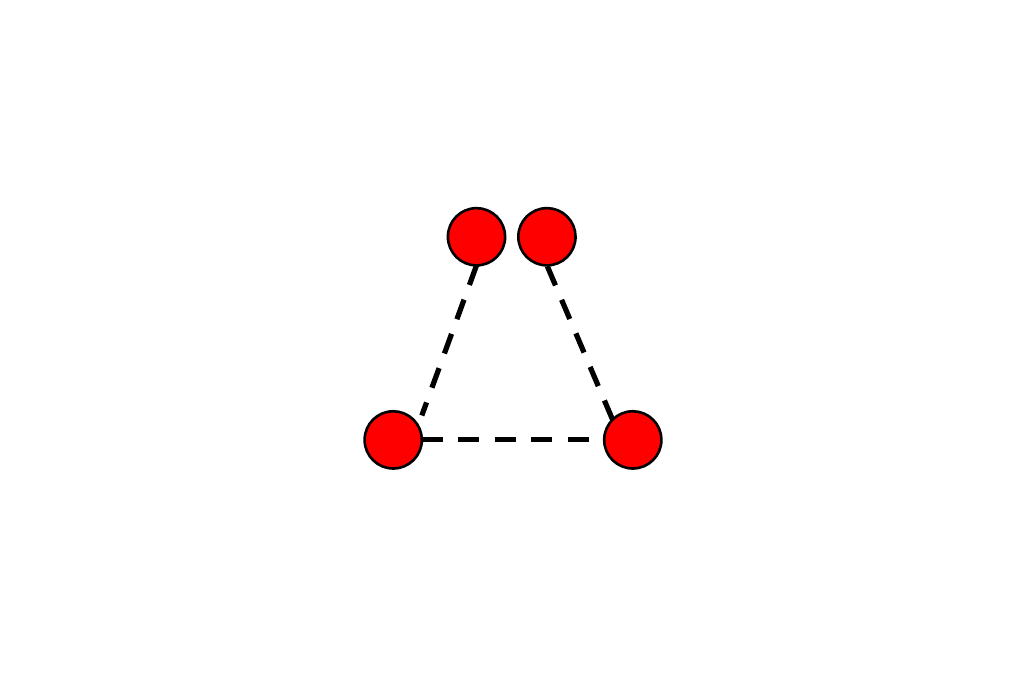}
  \caption{Small Separation}
  \label{fig:proof_1_c}
\end{subfigure}
\caption{
Mapping a toy graph (\ref{fig:proof_1_g}) into three different 2D-embedding spaces: (\ref{fig:proof_1_a}), (\ref{fig:proof_1_b}),  and (\ref{fig:proof_1_c}). The position of each node denotes the 2D-embedding vector and the dotted lines represent edges in $\cG$.}
\label{fig:toy}
\end{figure*}

\vspace{-3ex}
\subsection{GRL objective}
As a result of Theorem \ref{thm:1}, a GRL algorithm that only maximizes cohesion is not guaranteed to maximize separation and thus can lead to inferior embeddings such as the one shown in Figure \ref{fig:proof_1_c} for the toy graph. This is one of the major drawbacks of skip-gram based GRL algorithms that only attempt to maximize the similarity at nearby nodes (where neighborhood is defined using random walks). We posit this as a natural consequence of the origin of these algorithms in natural language processing (NLP) applications, where the definition of distances between words (and hence the separation) is not as straight-forward as in graphs. We thus present a generalized objective of GRL using both cohesion and separation. We can show that existing GRL algorithms (e.g., unigram negative sampling based approaches) optimize special cases of this GRL objective.

\begin{Definition}
\textbf{Generalized GRL Objective: } Given a graph $\cG$, the goal of a GRL algorithm is to optimize the following generalized objective function with respect to $\cZ$:
\begin{align}
    E(\cZ) =& ~ \coh(\alpha, \cZ) + \sep(\beta, \cZ) \nonumber\\
    =& \sum_{d=1}^{d_{max}} \sum_{(i,j) \in \cS_d} [\alpha_d ~ \sm(i,j) + \beta_d ~ \dsm(i,j)] \nonumber\\
    \text{such that,} & ~~ \frac{\alpha_1}{\beta_1} \gg 1, ~~ \frac{\alpha_{d_{max}}}{\beta_{d_{max}}} \ll 1, ~~ \text{and} ~~ \frac{\beta_{d_{max}}}{\beta_{1}} \gg 1 
    \nonumber
\end{align}
Note that we do not use a trade-off parameter between cohesion and separation  since any such parameter can be absorbed in $\alpha$ or $\beta$ as a constant multiplier. Different GRL algorithms optimize this generalized objective using different choices of similarity and dissimilarity functions, and settings of $\alpha$ and $\beta$ weights satisfying the GRL conditions in the above equation. From the perspective of separation, we would prefer a GRL algorithm that employs a large value of $\beta_{d_{max}}/\beta_1$, such that the dissimilarity at farthest node-pairs is substantially larger than that of the nearest node-pairs. We call this fraction $\beta_{d_{max}}/\beta_1$ as the \textbf{Separation Power} of a GRL algorithm.

\end{Definition}


\section{Proposed method}
\label{sec:method}
\paragraph{Negative sampling:}
Before we present our proposed GRL algorithm based on the ideas of cohesion and separation, we formally discuss the generic family of negative sampling algorithms of which our algorithm is a special case. The objective function of negative sampling is given by the following equation:



\begin{align}
    \max_{\cZ} \;\;\;\; \sum_{i\in\mathcal{V}} \sum_{j\in\mathcal{N}(i)}[\underbrace{\log\:\sigma(\bz_i^T \bz_j)}_\text{Positive Loss} + \nonumber\\
    {K} \sum_{k\in \mathcal{V}} \underbrace{P_{neg}(k|i)\log(\sigma(- \bz_i^T \bz_k))}_\text{Negative Loss}], \nonumber
\end{align}

where node-pair $(i, j)$ belongs to the positive corpus $D_{pos}$ while $(i, k)$ belongs to the negative corpus $D_{neg}$. 
We generally use random-walk strategy to construct $D_{pos}$, whereas $D_{neg}$ is constructed by sampling ${K}$ negative pairs $(i, k)$ for each positive sample $(i, j)$ with probability $P_{neg}(k|i)$ \cite{mikolov2013efficient, gutmann2010noise}.
A common choice of $P_{neg}(k|i)$ is the unigram distribution that samples $k$ with equal probability from all $n$ nodes, referred as the Unigram Negative Sampling (UNS) algorithm.
The objective function of UNS can be shown to be a special case of the generalized GRL objective where the similarity at nearby nodes (i.e., cohesion) corresponds to the positive loss while the dissimilarity at distant nodes (i.e., separation) corresponds to the negative loss. However, a major limitation with UNS is that the probability of sampling a negative node-pair is independent of the distance between the nodes. As a result, UNS pays equal importance to the dissimilarity of node pairs with varying distances in the calculation of separation, thus leading to poor separation power. Theorem \ref{thm:2} provides a formal analysis of the correspondence of UNS to the generalized GRL objective and shows that its separation power is equal to 1.


\begin{theorem}
Unigram Negative Sampling (UNS) Algorithm optimizes the generalized GRL objective with the following specifications:
    $\sm (i,j) = log(\sigma(\bz_i^T\bz_j))$, ~ $\dsm (i,j) = log(\sigma(-\bz_i^T\bz_j))$,
    $\alpha_d = \pi_d(C,\bA)$, where $\pi_d(C,\bA)$ is the probability of sampling a node-pair at distance $d$ using a $C$-length random walk on the graph with adjacency matrix $\bA$, and
    $\beta_d = KC/n.$
As a result, the Separation Power of UNS algorithm is equal to 1.
\label{thm:2}
\end{theorem}
\begin{proof}
Provided in the Supplementary material.
\end{proof}
\paragraph{Distance-aware Negative Sampler:}
We propose a Distance-aware Negative Sampler (DNS) which selects a negative sample $k$ for node $i$ using the sampling probability $P_{neg}(k|i)$, where $P_{neg}(k|i)$ is linearly proportional to the pair-wise distance $d(k, i)$. Formally,
\begin{align*}
    P_{neg}(k|i) &\propto d(k, i)\\
    P_{neg}(k|i) &= \frac{d(k, i)}{\mathcal{D}(i, \bA)},
\end{align*}
where $\mathcal{D}(i, \bA)$ is the sum of distance of all node-pairs that contain node $i$, $\mathcal{D}(i, \bA) = \sum_{s \in \cV}d(s, i)$. Let $\mathcal{D}(\bA)$ be equal to $\mathbb{E}_{i}(\mathcal{D}(i,\bA))$. Note that $\mathcal{D}(\bA)$ depends on the average degree of the graph.  For a fully connected graph, $\mathcal{D}(\bA) = n - 1$. On the other extreme, when the graph is a chain of $n$ nodes, then $\mathcal{D}(\bA) = \frac{n(n-1)}{2}$. Generally, since most real world graphs are sparse, $\mathcal{D}(\bA) \gg n - 1$.   
By construction, our proposed DNS approach has a separation power of $d_{max}$ as stated in Theorem \ref{thm:3}.

\begin{theorem}
Distance-aware Negative Sampling (DNS) Algorithm optimizes the generalized GRL objective with the following specifications:
    $\sm (i,j) = log(\sigma(\bz_i^T\bz_j))$, ~ $\dsm (i,j) = log(\sigma(-\bz_i^T\bz_j))$,
    $\alpha_d = \pi_d(C,\bA)$, where $\pi_d(C,\bA)$ is the probability of sampling a node-pair at distance $d$ using a $C$-length random walk on the graph with adjacency matrix $\bA$, and
    $\beta_d = KCd/\mathcal{D}(\bA).$
As a result, the Separation Power of DNS algorithm is equal to $d_{max}$.
\label{thm:3}
\end{theorem}

\begin{proof}
Provided in the Supplementary material.
\end{proof}

\begin{corollary}
For UNS, $(\frac{\alpha_d}{\beta_d})_{UNS} = \frac{\pi_d(C,\bA)n}{KC}$ and for DNS, $(\frac{\alpha_d}{\beta_d})_{DNS} = \frac{\pi_d(C,\bA)\mathcal{D}(\bA)}{KCd}$. Hence, 
$(\frac{\alpha_d}{\beta_d})_{UNS} < (\frac{\alpha_d}{\beta_d})_{DNS}$ when $n < \frac{\mathcal{D}(\bA)}{d}$. 
\end{corollary}
The above  corollary helps us understand useful operating points of DNS. Since  $\mathcal{D}(\bA) \gg n$ for most graphs, the $(\alpha/\beta)$ ratio is generally always larger for DNS than UNS. We have also empirically observed that the $(\alpha/\beta)$ ratio increases with $C$ for all graphs considered in this work. As a result, DNS operates better at lower values of $C$ since $(\alpha/\beta)$ ratios remain small for moderate values of $d$.
Additionally, we have empirically observed that DNS works better for sparse graphs since there is a larger spread in the network distances across all node-pairs, making it possible for DNS to maximize separation in the embedding space for distant nodes.

The embedding space learned by DNS indeed preserves the graph-based similarity structure among nodes. Formally, Theorem \ref{thm:4} shows that the pairwise similarity in embedding space is a function of node-pair distance and for negative node-pairs, the similarity is inversely proportional to the distance. 

\begin{theorem}
Let the average pairwise similarity for any two nodes at distance d be given by $\xi_d$ = $\frac{1}{|\cS_{d}|}\sm(i, j) = \frac{1}{|\cS_{d}|} \sum_{(i,j)\in \cS_d} \sigma (z_i^Tz_j)$.  
We can then show that DNS generates embeddings such that $\xi_d$ is a function of $d$ and for $d > C$, $\xi_{d}$ is inversely proportional to $d$.
\label{thm:4}
\end{theorem}
\begin{proof}
Provided in the Supplementary material.
\end{proof}

While negative sampling with linearly proportional distances is a simple heuristic, we can have a more general form of DNS by adding super-linearity or sub-linearity in the negative sampling probability which is,
$
    P_n(r|u) \propto (d(r, u)) ^ {\gamma}.
$
Here $\gamma$ is a hyper-parameter and we can vary $\gamma$ based on the properties of the dataset.

\begin{table*}[t]
  \caption{Summary statistics of the datasets we used for experiments where we choose the largest connected components from 390 components for CiteSeer and 78 components for Cora (PubMed and PPI are single connected component graph). We represent the largest component as $\mathcal{G=(V, E)}$ and the set of unique class labels as $y$. PPI dataset has 121 classes with binary labels. The average node degree is represented by $\overline{deg}$, and the maximum node pair distance is denoted by $d_{MAX}$.}
  \label{table:datasets}
  \centering
     \begin{tabular}{cccccccc}
    \toprule
    Stat                         & CiteSeer         & Cora       & PubMed     & PPI         & Syn. Sparse      & Syn. Moderate      & Syn. Dense \\
    \midrule
    $|\mathcal{V}|$              & $ 2,120 $       & $2,485$   & $19,717$  & $2,339$   & $2,000$        & $2,000$          & $2,000$        \\
    $|\mathcal{E}|$              & $ 7,358 $       & $ 10,138 $   & $88,648$  & $65,430$   & $4,982$        & $12,062$          & $30,472$        \\
    $|y|$                        & $ 6 $       & $7$   & $3$  & $121$   & $7$        & $5$          & $4$        \\
    $d_{MAX}$                    & $ 28 $       & $19$   & $18$  & $7$   & $106$        & $82$          & $4$        \\
    $\overline{deg}$             & $3.47$       & $4.08$   & $4.5$  & $27.97$   & $2.49$        & $6.03$          & $15.24$        \\
    \bottomrule
  \end{tabular}
  \vspace{-2ex}
\end{table*}

\paragraph{Complexity analysis:}
DNS requires pairwise shortest distance computation for all node pairs as a preprocessing step. The time complexity to compute all-pair shortest path lengths is $\Theta(nm + n^{2}log n)$ \cite{dijkstra1959note, 715934} and the space complexity is $\cO(n^{2})$ to store the normalized probabilities for all node pairs. While there are efficient techniques to precompute shortest distances that can be coupled with our approach \cite{6690045,cohen2003reachability, jin20093}, 
our basic DNS based model would still require $\cO(n^{2})$ space to store the normalized probabilities, which is not scalable.


\paragraph{Scalable DNS Approach:}
We develop an approach that reduces the space complexity of our DNS model without increasing the training time complexity. In this approach, we compress the pairwise distance matrix during preprocessing and decompress it back during training. We use some landmark nodes to store the shortest distances of node-to-landmark and landmark-to-landmark at the preprocessing step; subsequently, we decompress the information to reconstruct the node pairwise shortest distance matrix for minibatch nodes at the training phase. 

\begin{figure}
    \centering
    \includegraphics[width=7cm,height=7cm,keepaspectratio]{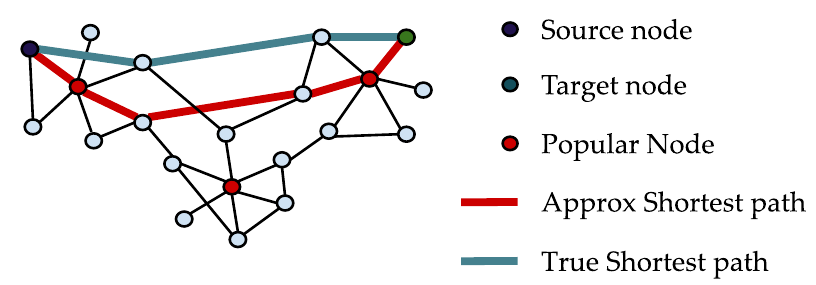}
    \caption{DNS-approx heuristic.}
    \label{fig:heu1}
    \vspace{-1ex}
\end{figure}

\begin{figure}
    \centering
    \includegraphics[width=7cm,height=7cm,keepaspectratio]{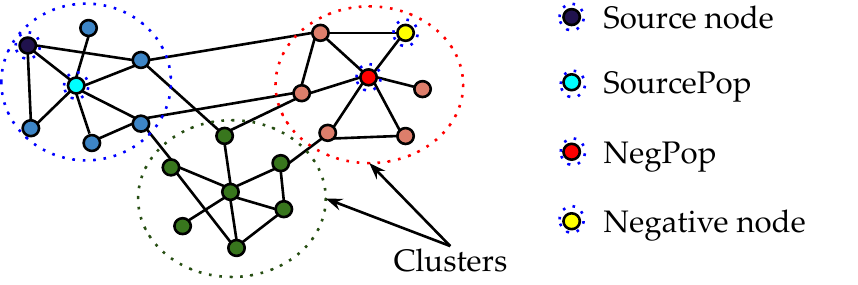}
    \vspace{-2ex}
    \caption{DNS-scalable heuristic.}
    \label{fig:heu2}
    \vspace{-3ex}
\end{figure}




Similar to the core-net approach, we select the landmark nodes as popular nodes that have higher degrees than a degree-threshold. We hypothesize our heuristic based on the intuition that the higher degree nodes are more likely on the shortest path of any two nodes. Consequently, we store the shortest path distances of all nodes to their closest popular nodes in a node-to-popular (N2P) vector and the shortest distances between popular nodes in a popular-to-popular (P2P) matrix. At the training phase, we reconstruct the distance profile of every minibatch node (source) by adding the N2P distance to its popular node, the P2P distances across popular nodes, and the N2P distance from popular nodes to target nodes (see Figure \ref{fig:heu1}). We denote this heuristic as DNS-approx as it is approximating the true shortest path between node-pairs with paths going through popular nodes. 

DNS-approx requires $\cO(p^{2} + n)$ space, where $p$ is the number of popular nodes (in general, we select 10\% nodes as popular nodes), and $n$ is the number of nodes in the graph. The degree-threshold, which decides the popular nodes, controls this space requirement of this heuristic as a higher degree threshold selects a lower number of popular nodes and vice versa. Although DNS-approx can reduce the space complexity by a large margin, we still need to reconstruct the approximate distance profile of a minibatch node with all other nodes, which costs us extra runtime during model training. To address this concern, we need a solution that does not need to reconstruct the approximate distances of a minibatch node with all other nodes and still perform distance-aware negative sampling. 


We propose another scalable DNS approach, DNS-scalable, which only uses the popular-to-popular matrix to sample negative-popular (NegPop) nodes where the probability of selecting a negative-popular node of another popular node is linearly proportional to its distance from the popular node of the source node (SourePop) node. 
Every NegPop node corresponds to a cluster of negative nodes based on the N2P vector mapping. Finally, to sample negative nodes, we uniformly sample a node from the cluster of the selected NegPop node (see Figure \ref{fig:heu2}).

\section{Evaluation setup}
\label{sec:experiments}

\paragraph{Benchmark datasets:}
In our experiments, we use four benchmark datasets for node classification: CiteSeer, Cora, PubMed, and PPI \cite{DBLP:journals/corr/YangCS16, DBLP:journals/corr/ZitnikL17} where CiteSeer, Cora, and PubMed are citation-networks and PPI is protein-protein interaction network. 
In the citation network, the nodes correspond to articles of different subjects, whereas the edges correspond to citations between those articles; consequently, the node prediction task on this network is to predict the article subject. Meanwhile, the physical interaction between different proteins with their defined roles (cellular functions) on a specific human tissue 
\footnote{Instead of working on multiple graphs, we randomly select one PPI network corresponding to a specific human tissue.} 
is represented using a protein-protein interaction network where the classification task is to predict the protein roles. 

Since our goal is to find meaningful node-embeddings of a graph that only reflect the graph structure information rather than the node feature information, we do not use any node features of these benchmark datasets for our experiments.
These datasets have multiple small disconnected components with the largest connected component that describes the graph structure properly. Consequently, our proposed DNS sampler requires a definite distance between any node pairs; therefore, we focus our experiments on the largest connected component of these networks. 

Table \ref{table:datasets} summarizes all the datasets with significant statistics where we see CiteSeer, Cora, and PubMed are sparse datasets with average degrees from 3.47 to 4.5, whereas PPI is quite dense with average degree 27.97. 

\textbf{Synthetic datasets.} We further analyze the representation quality of DNS-based GRL models with varying graph density using three synthetic datasets- sparse, moderate, and dense.
We provide in detail description of synthetic data generation in the supplementary material.

\paragraph{Baseline models:}
For baseline models, we choose DeepWalk \cite{perozzi2014deepwalk}, node2vec \cite{grover2016node2vec}, LINE \cite{tang2015line}, GAE \cite{kipf2016variational} and VGAE \cite{kipf2016variational} (details in Section \ref{subsection:baselines}). 
We implement the Unigram Negative Sampler and the Distance-aware Negative Sampler on DeepWalk and node2vec models. For the DeepWalk model, we also implement the UNS-deg formulation described in \ref{subsection:Negative_Sampling}. Further on the DeepWalk model, we also implement the DNS-approx and the DNS-scalable method. Note that our DNS sampler does not depend on the DeepWalk method, and we can pair DNS with any other GRL approach like GraphSAGE and GCN. We describe the hyperparameter settings
and the node classification setup in details in the supplementary materials.

\paragraph{Evaluation Metrics:} We use F1-Macro to report the classification accuracy on node classification tasks. Moreover, we visualize the node representations using standard visualization tools like t-SNE, which is a dimensionality reduction technique that preserves local similarities.

\begin{table*}[t]
  \caption{The summary of the model performances in terms of downstream node classification F1-macro score. We highlight the best score for each dataset. For Cora, CiteSeer, PPI, and PubMed, we choose context window 4 to report the results. Both DW-DNS-approx and DW-DNS-scalable use 10\% nodes as popular nodes. We run each model 5 times and report the performances in terms of mean and standard deviations. }
  \label{table:f1-benchmark}
  \centering
  \begin{tabular}{lllll}
    \toprule
    & \multicolumn{4}{c}{Dataset}                   \\
    \cmidrule(r){2-5}
    Models     & CiteSeer     & Cora     & PubMed     & PPI \\
    \midrule
    GAE              & $0.40 \pm 0.01$       & $0.61 \pm 0.02$   & $0.59 \pm 0.02$  & $0.68\pm  0.00$        \\
    VGAE             & $0.39 \pm 0.02$       & $0.58 \pm 0.02$   & $0.60 \pm 0.02$  & $0.67\pm  0.00$        \\
    LINE             & $0.37 \pm 0.05$       & $0.52 \pm 0.05$   & $0.47 \pm 0.07$  & $0.68\pm  0.00$        \\
    node2vec-UNS     & $0.43 \pm 0.02$       & $0.54 \pm 0.01$   & $0.56 \pm 0.01$  & $0.63\pm  0.00$        \\
    \textbf{node2vec-DNS}     & $0.52 \pm 0.01$       & $0.62\pm  0.00$   & $0.58 \pm 0.01$  & $0.64\pm  0.00$        \\
    DeepWalk-UNS     & $0.51 \pm 0.01$       & $0.67\pm  0.00$   & $0.58\pm  0.00$  & $\mathbf{0.69 \pm  0.00}$        \\
    DeepWalk-UNS-deg & $0.47\pm  0.00$       & $0.65 \pm 0.01$   & $0.54\pm  0.00$  & $0.68\pm  0.00$        \\
    \textbf{DeepWalk-DNS}     & $\mathbf{0.61}\pm \mathbf{0.01}$       & $\mathbf{0.72 \pm 0.01}$   & $ 0.63\pm  0.00$  & $\mathbf{0.69\pm  0.00}$        \\
    \textbf{DW-DNS-approx} & $0.59 \pm 0.01$ & $0.71 \pm 0.01$ & $\mathbf{0.64 \pm  0.00}$ & $\mathbf{0.69 \pm  0.00}$ \\
    \textbf{DW-DNS-scalable} & $0.57 \pm 0.01$ & $0.70 \pm 0.01$ & $0.63 \pm 0.01$ & $0.68 \pm  0.00$ \\
    \bottomrule
  \end{tabular}
\end{table*}
\section{Results}
\begin{figure*}[t]
\begin{subfigure}{.33\textwidth}
  \centering
  \includegraphics[width=.99\linewidth]{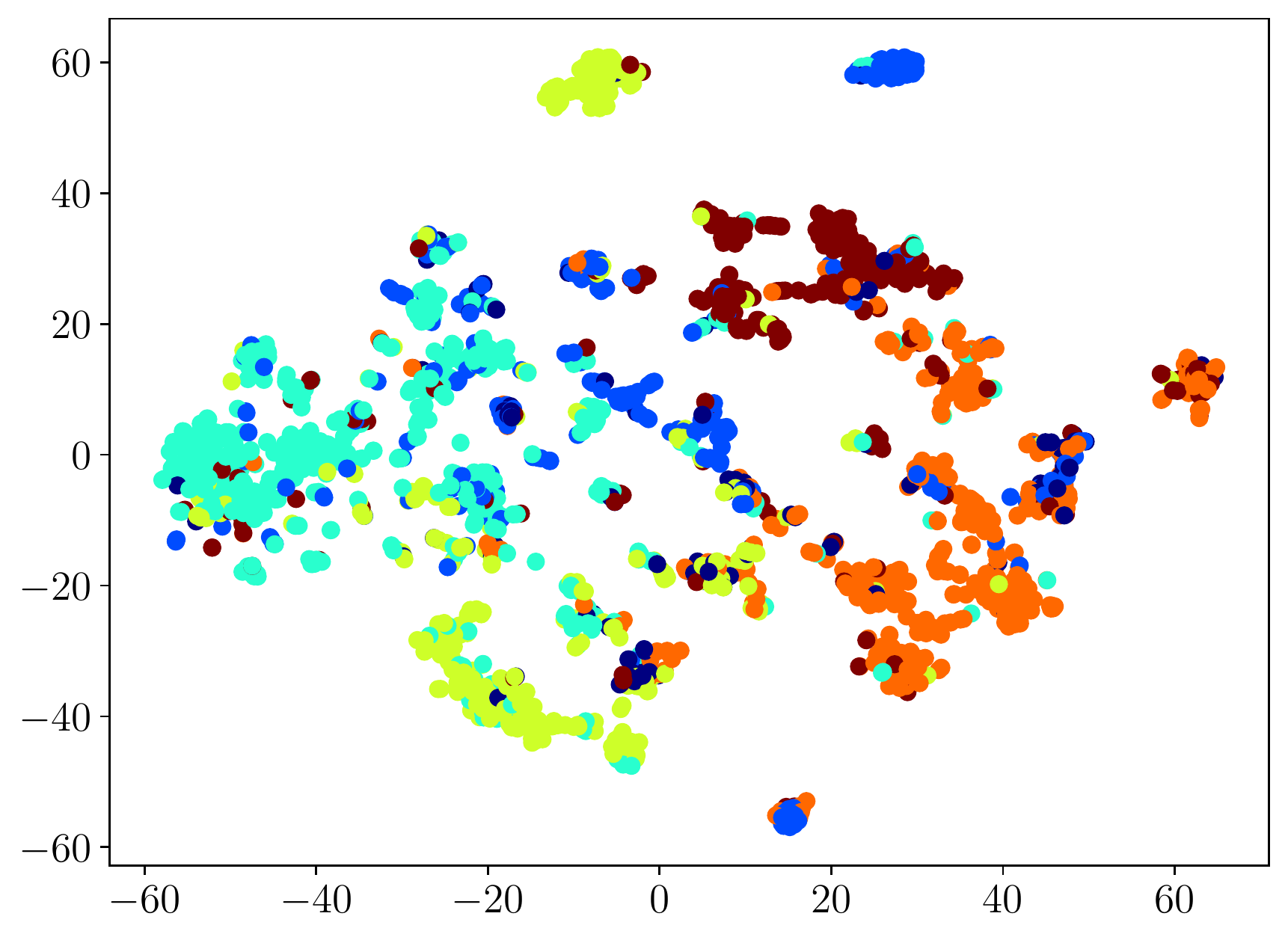}
  \caption{DeepWalk-DNS}
  \label{fig:tsne-DNS}
\end{subfigure}
\begin{subfigure}{.33\textwidth}
  \centering
  \includegraphics[width=.99\linewidth]{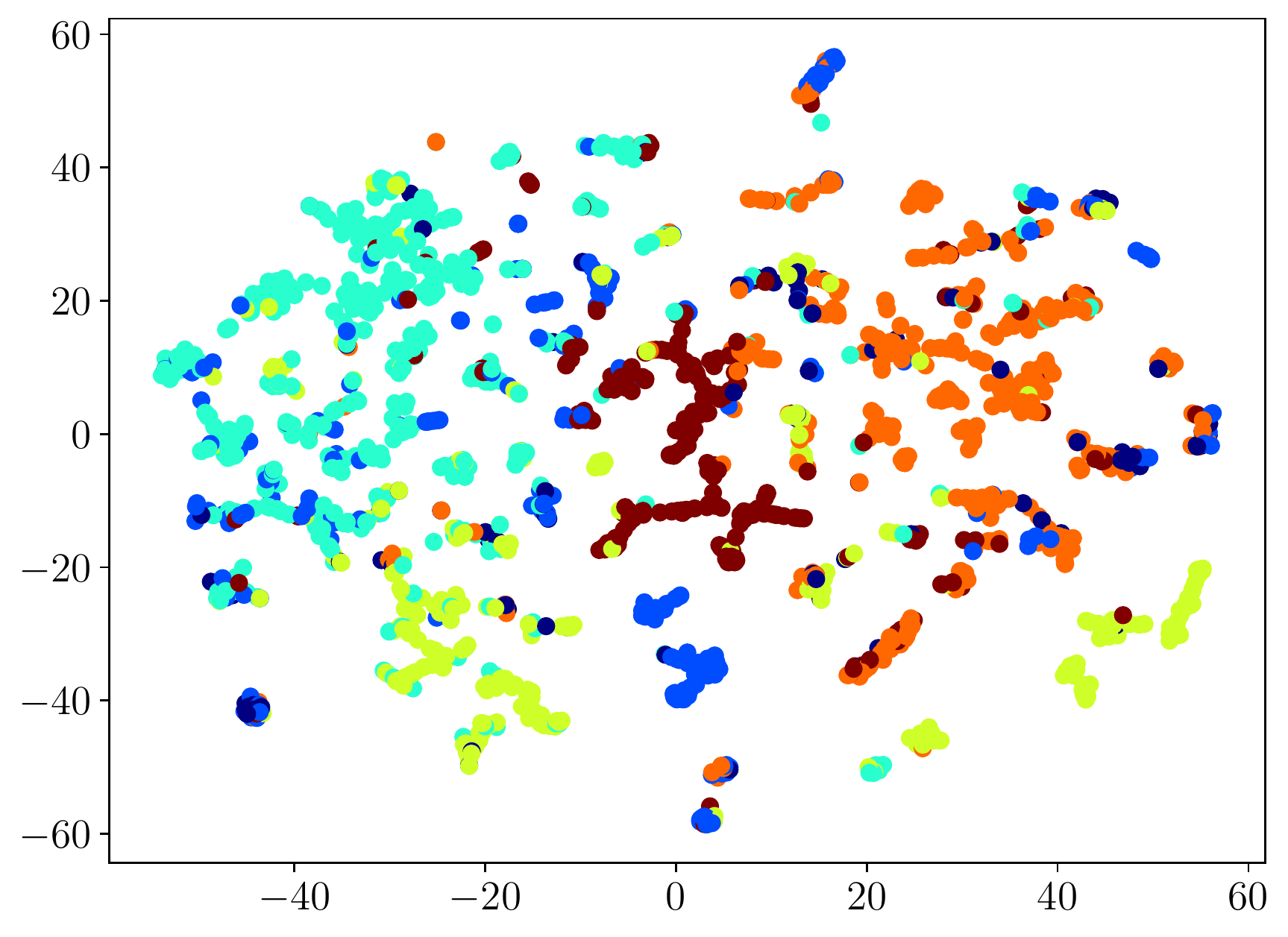}
  \caption{DeepWalk-UNS}
  \label{fig:tsne-UNS}
\end{subfigure}
\begin{subfigure}{.33\textwidth}
  \centering
  \includegraphics[width=.99\linewidth]{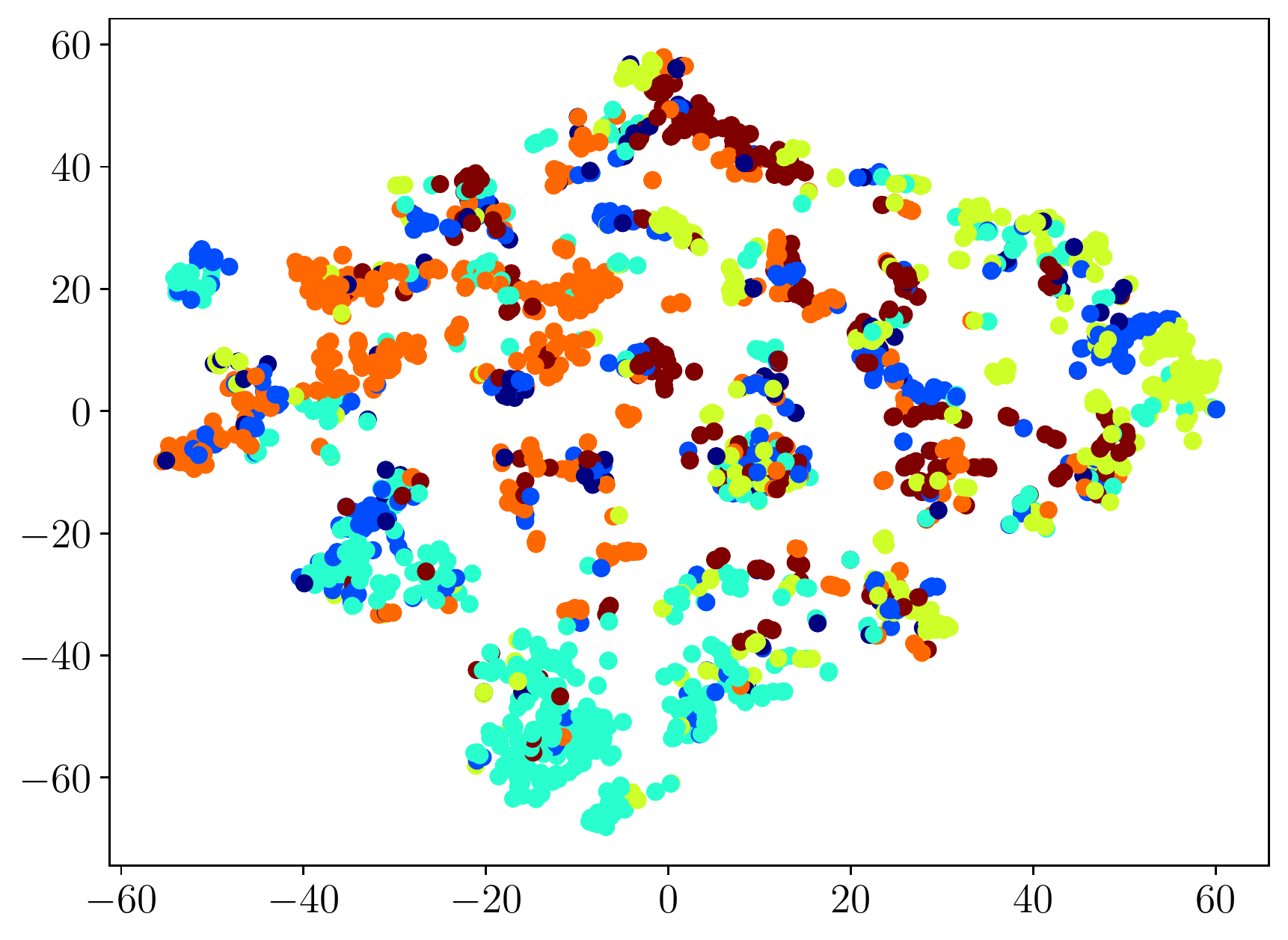}
  \caption{GAE}
  \label{fig:tsne-GAE}
\end{subfigure}
\caption{
t-SNE plot for embeddings generated by DeepWalk with Distance-aware Negative Sampler model (DeepWalk-DNS), DeepWalk with Unigram Negative Sampler model (DeepWalk-UNS), and Graph Auto Encoder model (GAE) on CiteSeer dataset.}
\label{fig:tsne-citeseer}
\vspace{-2ex}
\end{figure*}

\begin{figure*}[t]
\begin{subfigure}{.33\textwidth}
  \centering
  \includegraphics[width=.99\linewidth]{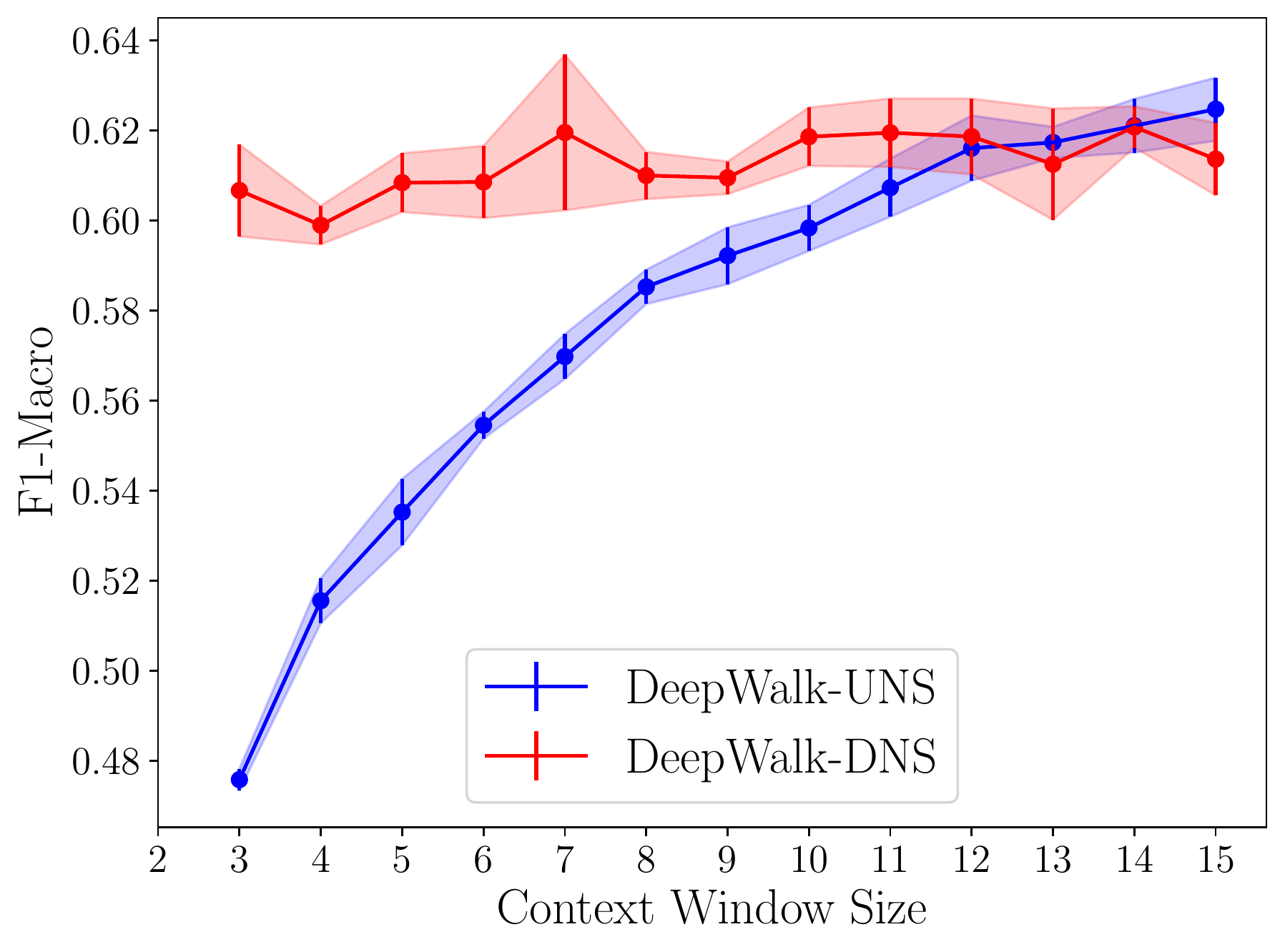}
  \caption{CiteSeer}
  \label{fig:f1-context-citesser}
\end{subfigure}
\begin{subfigure}{.33\textwidth}
  \centering
  \includegraphics[width=.99\linewidth]{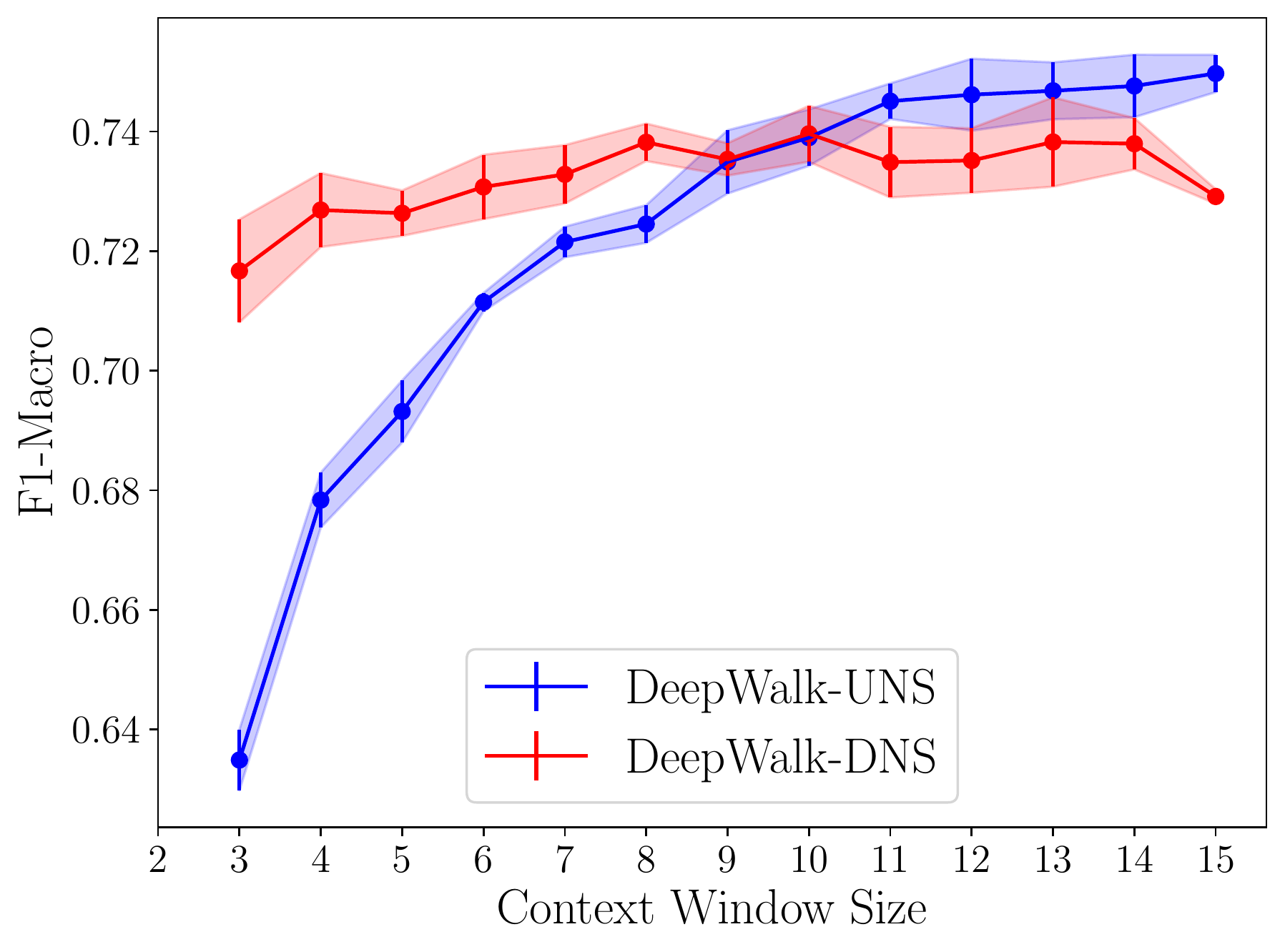}
  \caption{Cora}
  \label{fig:f1-context-cora}
\end{subfigure}
\begin{subfigure}{.33\textwidth}
  \centering
  \includegraphics[width=.99\linewidth]{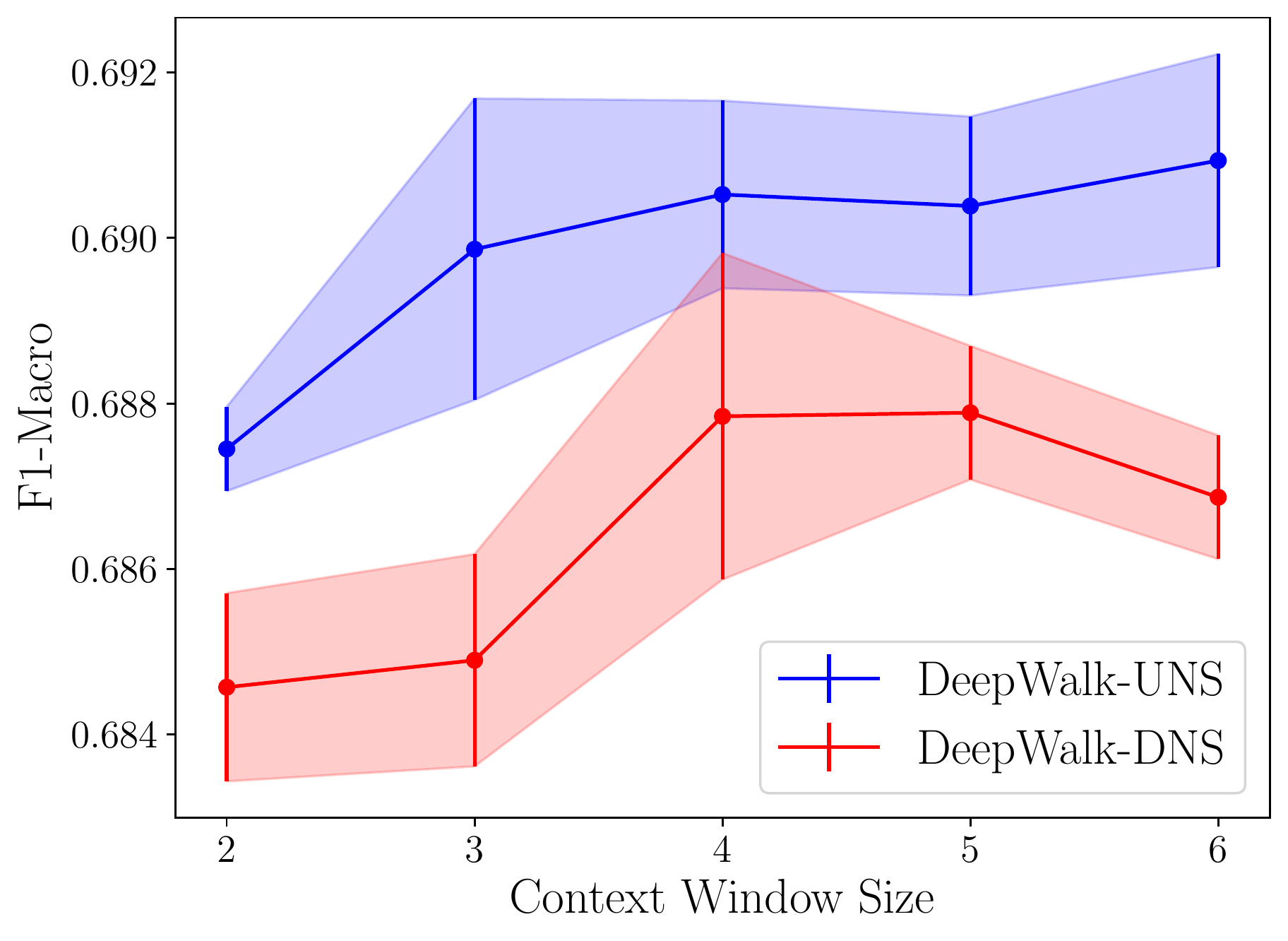}
  \caption{PPI}
  \label{fig:f1-context-ppi}
\end{subfigure}
\caption{
Node classification performance (measured by F1-Macro score) plot with varying context window on CiteSeer, Cora, and PPI dataset. DeepWalk with Distance-aware Negative Sampler (DeepWalk-DNS) and with Unigram Negative Sampler (DeepWalk-UNS) are the competing models.}
\label{fig:f1-context}
\end{figure*}

\begin{figure*}[t]
\begin{subfigure}{.33\textwidth}
  \centering
  \includegraphics[width=.99\linewidth]{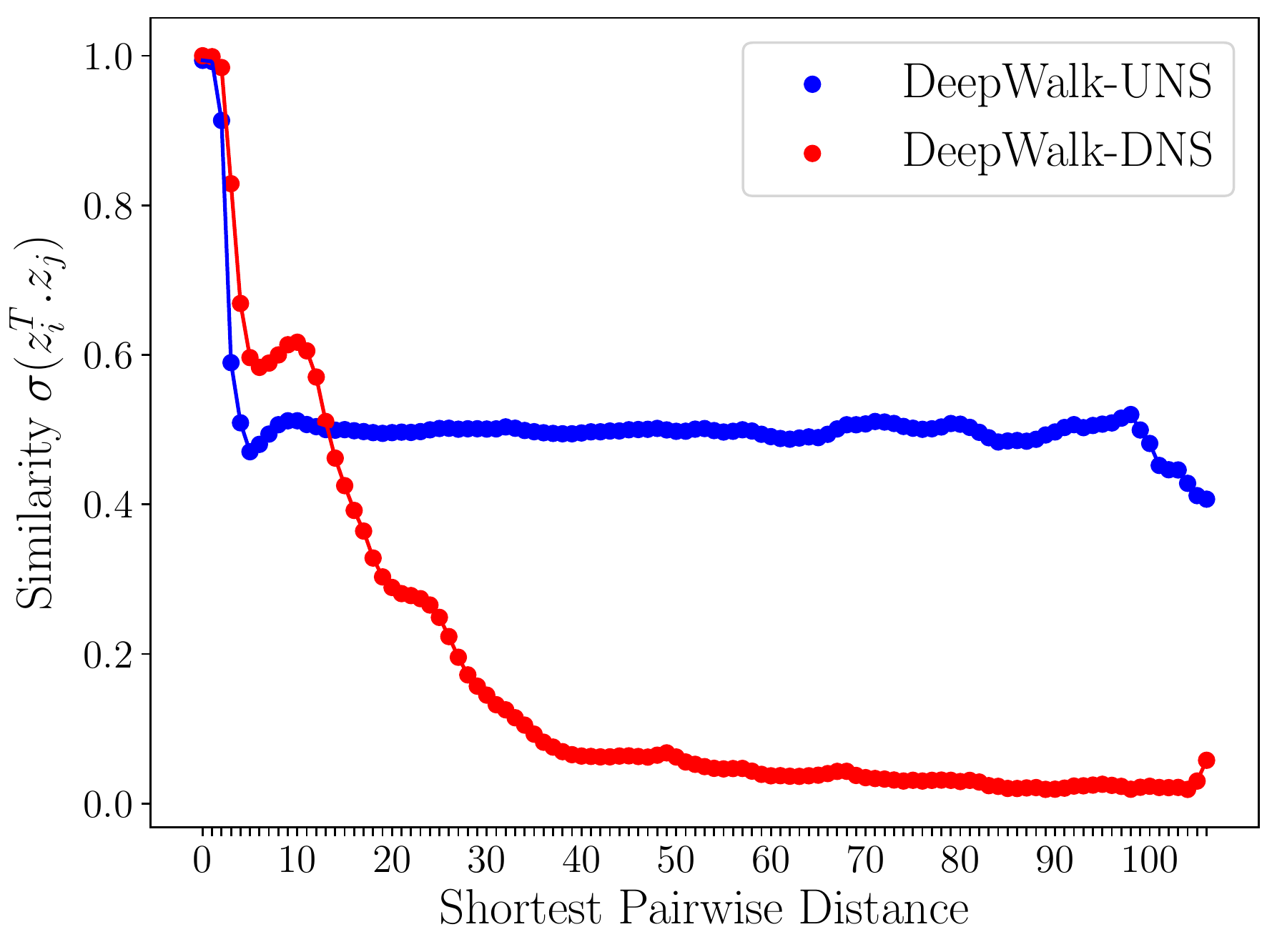}
  \caption{Synthetic Sparse}
  \label{fig:sim-sparse}
\end{subfigure}
\begin{subfigure}{.33\textwidth}
  \centering
  \includegraphics[width=.99\linewidth]{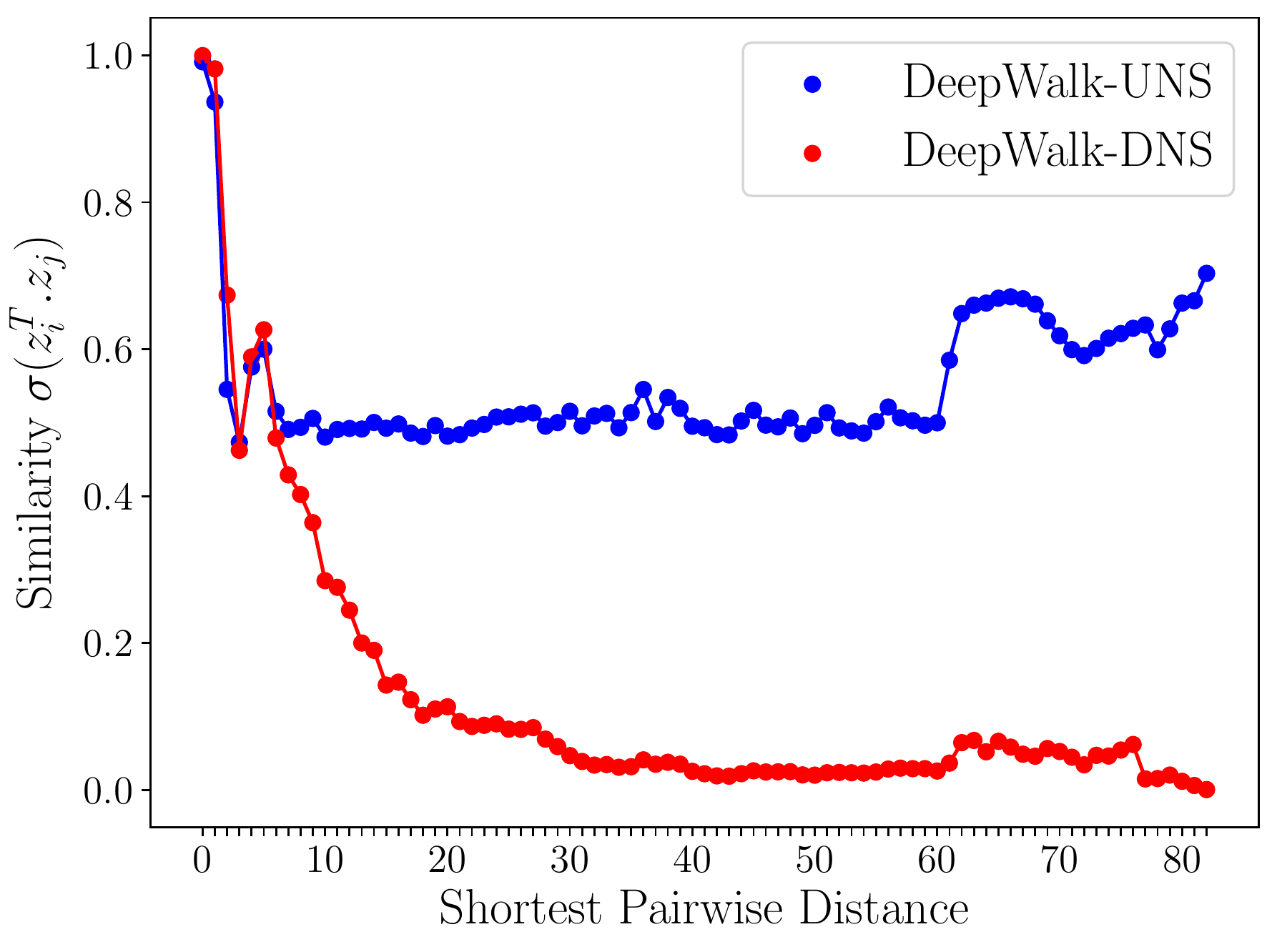}
  \caption{Synthetic Moderate}
  \label{fig:sim-moderate}
\end{subfigure}
\begin{subfigure}{.33\textwidth}
  \centering
  \includegraphics[width=.99\linewidth]{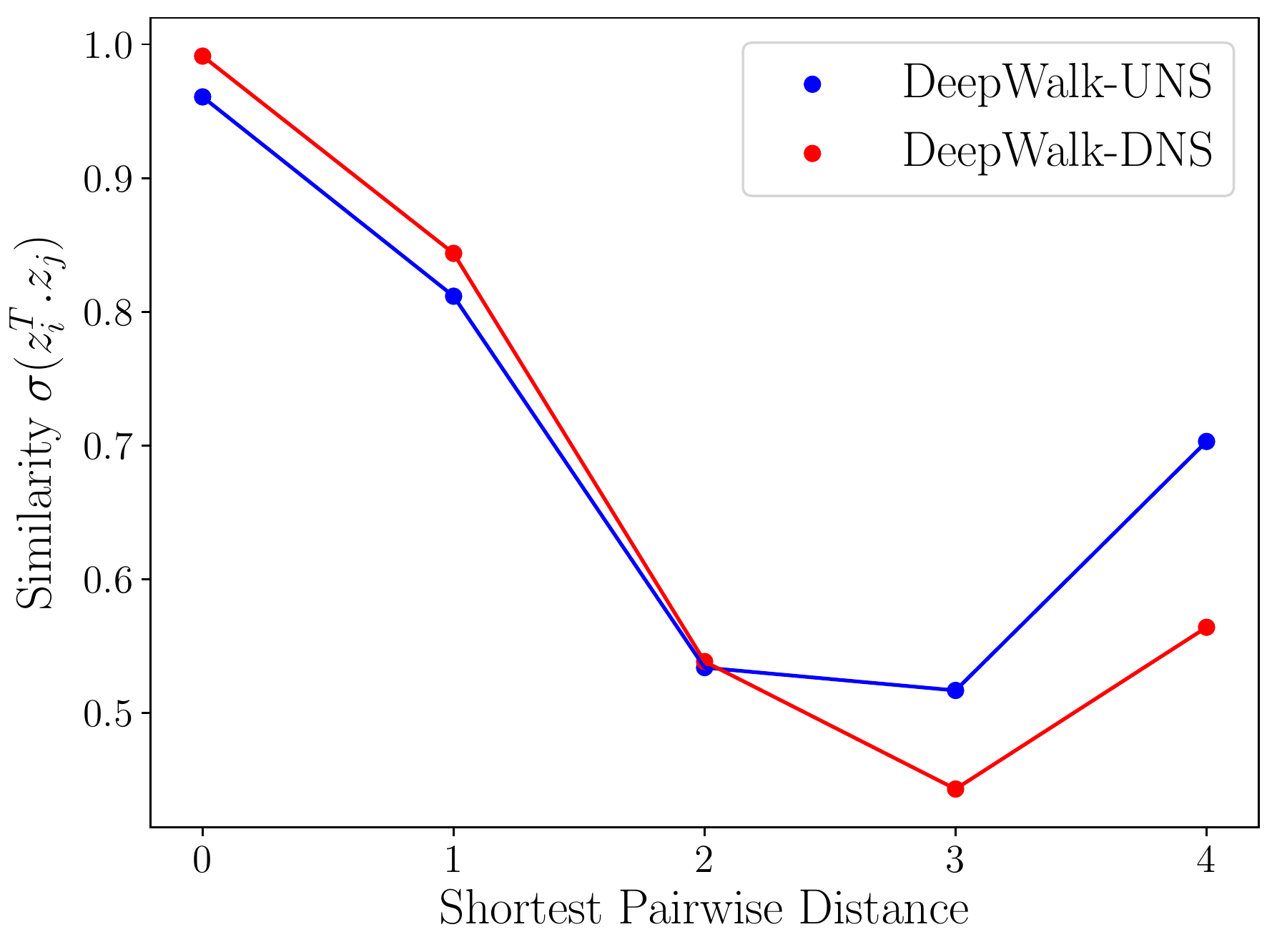}
  \caption{Synthetic Dense}
  \label{fig:sim-dense}
\end{subfigure}
\caption{
Average Pairwise Similarity of all node-pairs in embedding space where similarity = $\sigma(z_i^Tz_j)$ for $z_i$ and $z_j$ node embeddings. Embeddings generated by DNS based GRL model show minimum similarity for distant nodes with the similarity decreasing with increasing distance $d$.
}
\label{fig:sim-synthetic}
\end{figure*}

Table \ref{table:f1-benchmark} compares the performance of our proposed DNS-based GRL models (DeepWalk-DNS and node2vec-DNS) and DNS-based scalable approaches (DW-DNS-approx and DW-DNS-scalable) with other baseline models on the benchmark node classification tasks. From the results, DNS-based models show a significant improvement in the F1-Macro score than that of the traditional sampling-based models across all benchmark datasets. Further, our scalable DNS models show comparable performance as the basic DNS models. The t-SNE plot (Figure  \ref{fig:tsne-citeseer}) shows that the DNS-based model learns more meaningful feature visualizations with better cohesion and separation of the classes (shown using colors)  than that of the other models for CiteSeer dataset.

To measure the impact of varying context windows, Figure \ref{fig:f1-context} shows F1 score of DeepWalk-DNS and DeepWalk-UNS with varying context windows on CiteSeer, Cora, and PPI datasets, where the performance of DNS-based methods tends to get closer to UNS-based methods with increasing context window, which is in-line with our discussion in Section 4. However, we can see that the F1-score of DNS is significantly larger than that of UNS for a large range of context windows smaller than a reasonable value. In practice, we prefer low context windows during negative sampling for better optimization time at learning phase \cite{grover2016node2vec}. Moreover, for dense graphs, such as PPI, dissimilar nodes have low pairwise distances that weaken our node-similarity assumption and decrease DNS-based model performance. However, we can set the value of $\gamma$ to a small value in $\gamma$-linear negative sampling, which reduces the effect of distances and improves performance for the densely connected graphs. We demonstrate the sensitivity of the performance of $\gamma$-DNS to  $\gamma$ in Supplementary Materials. We also perform ablation studies to understand the importance of distance-aware negative sampling probabilities in comparison to baseline methods in Supplementary Materials.
Figure \ref{fig:sim-synthetic} shows the similarity of the embeddings generated by DNS-based models on the synthetic graphs, which is inversely proportional to the pairwise distance that maximizes the separation of distant node pairs. A detailed analysis of DeepWalk-DNS on the synthetic dataset and the effect of $\gamma$ in $\gamma$-linear negative sampling is in Supplementary Materials.

\begin{table}[t]
    \caption{Comparison of scalable heuristics on PubMed dataset. We denote the model size as \textit{size}, training time per epoch as \textit{time} with second as unit, and F1-Macro score of the downstream node classification task as \textit{acc}.}
    \label{tab:scalable}
    \centering
    \begin{tabular}{cccc}
    \toprule
     & \textit{size} & \textit{time} & \textit{acc}\\
    \midrule
    DNS-approx $(p=0.1n)$  & 43.7 MB & 1216 &0.64\\
    DNS-approx $(p=0.3n)$ & 256.4 MB & 1357 &0.63\\
    DNS-approx $(p=0.5n)$ & 871.1 MB & 1556 &0.63\\
    \midrule
    DNS-scalable $(p=0.1n)$  & 66 MB & 108  &0.63\\
    DNS-scalable $(p=0.3n)$ & 239.7 MB & 281 &0.64\\
    DNS-scalable $(p=0.5n)$ & 639.9 MB & 655 &0.64\\
    \midrule
    Basic DNS & 1.5 GB & 90 & 0.63\\
    \bottomrule
    \end{tabular}
\end{table}

Table \ref{tab:scalable} shows the comparison of DNS-approx and DNS-scalable with the DNS approach in terms of the model size, the per epoch training time, and the downstream node classification accuracy on PubMed dataset. The space complexity of both heuristics is $\cO(p^2 + n)$. 
From the table, we see that both heuristics take much less space than the basic DNS approach, even when we consider 50\% of nodes as popular nodes. We save all the trained models  and the size of the saved model objects is used as a measure of space complexity. We see the model sizes are in a similar range for both heuristics when they have an equal number of popular nodes. However, we see a large reduction in the training time for the DNS-scalable method since it does not reconstruct the shortest distance matrix, like DNS-approx. The training time for the DNS-scalable approach increases with the number of popular nodes because we select negative-popular nodes from the popular matrix, and the number of selection operations increases with the popular matrix size. All the models perform similarly on the downstream node classification task, which shows that we can get similar performance from our DNS-scalable method with a much smaller space complexity (with a small number of popular nodes) and a comparable training time complexity. 
\vspace{-2ex}
\section{Conclusions and Future Work}
This work presents a detailed discussion on Distance-aware Negative Sampling (DNS) for unsupervised Graph Representation Learning (GRL) where the node representations reflect graph structure better than the existing GRL methods by maximizing cohesion and separation on different sizes of networks. With theoretical analysis on cohesion and separation and empirical results on the connected components of benchmark datasets, we present DNS sampler as state-of-the-art sampler that better optimizes the negative-sampling objective on unsupervised GRL. Future directions of this research could focus on the scalable DNS with disconnected components.
\vspace{-2ex}
\section*{Acknowledgments}
This work was supported by NSF grant \#1940247.
\vspace{-3ex}

\bibliographystyle{unsrt}

\small{
\bibliography{reference}

\begin{thebibliography}{10}

\bibitem{DBLP:journals/corr/YangCS16}
Zhilin Yang, William~W. Cohen, and Ruslan Salakhutdinov.
\newblock Revisiting semi-supervised learning with graph embeddings.
\newblock {\em CoRR}, abs/1603.08861, 2016.

\bibitem{DBLP:journals/corr/KipfW16}
Thomas~N. Kipf and Max Welling.
\newblock Semi-supervised classification with graph convolutional networks.
\newblock {\em CoRR}, abs/1609.02907, 2016.

\bibitem{DBLP:journals/corr/GilmerSRVD17}
Justin Gilmer, Samuel~S. Schoenholz, Patrick~F. Riley, Oriol Vinyals, and
  George~E. Dahl.
\newblock Neural message passing for quantum chemistry.
\newblock {\em CoRR}, abs/1704.01212, 2017.

\bibitem{DBLP:journals/corr/abs-1710-03059}
Alberto Garc{\'{\i}}a{-}Dur{\'{a}}n and Mathias Niepert.
\newblock Learning graph representations with embedding propagation.
\newblock {\em CoRR}, abs/1710.03059, 2017.

\bibitem{hamilton2017inductive}
Will Hamilton, Zhitao Ying, and Jure Leskovec.
\newblock Inductive representation learning on large graphs.
\newblock In {\em NeurIPS}, pages 1024--1034, 2017.

\bibitem{velivckovic2017graph}
Petar Veli{\v{c}}kovi{\'c}, Guillem Cucurull, Arantxa Casanova, Adriana Romero,
  Pietro Lio, and Yoshua Bengio.
\newblock Graph attention networks.
\newblock {\em arXiv preprint arXiv:1710.10903}, 2017.

\bibitem{mikolov2013efficient}
Tomas Mikolov, Kai Chen, Greg Corrado, and Jeffrey Dean.
\newblock Efficient estimation of word representations in vector space.
\newblock {\em arXiv preprint arXiv:1301.3781}, 2013.

\bibitem{ahmed2013distributed}
Amr Ahmed, Nino Shervashidze, Shravan Narayanamurthy, Vanja Josifovski, and
  Alexander~J Smola.
\newblock Distributed large-scale natural graph factorization.
\newblock In {\em WWW}, pages 37--48, 2013.

\bibitem{cao2015grarep}
Shaosheng Cao, Wei Lu, and Qiongkai Xu.
\newblock Grarep: Learning graph representations with global structural
  information.
\newblock In {\em CIKM}, pages 891--900, 2015.

\bibitem{ou2016asymmetric}
Mingdong Ou, Peng Cui, Jian Pei, Ziwei Zhang, and Wenwu Zhu.
\newblock Asymmetric transitivity preserving graph embedding.
\newblock In {\em KDD}, pages 1105--1114, 2016.

\bibitem{perozzi2014deepwalk}
Bryan Perozzi, Rami Al-Rfou, and Steven Skiena.
\newblock Deepwalk: Online learning of social representations.
\newblock In {\em KDD}, pages 701--710, 2014.

\bibitem{DBLP:journals/corr/abs-1710-09599}
Sami Abu{-}El{-}Haija, Bryan Perozzi, Rami Al{-}Rfou, and Alex Alemi.
\newblock Watch your step: Learning graph embeddings through attention.
\newblock {\em CoRR}, abs/1710.09599, 2017.

\bibitem{armandpour2019robust}
Mohammadreza Armandpour, Patrick Ding, Jianhua Huang, and Xia Hu.
\newblock Robust negative sampling for network embedding.
\newblock In {\em AAAI}, volume~33, pages 3191--3198, 2019.

\bibitem{kipf2016variational}
Thomas~N Kipf and Max Welling.
\newblock Variational graph auto-encoders.
\newblock {\em arXiv preprint arXiv:1611.07308}, 2016.

\bibitem{velivckovic2018deep}
Petar Veli{\v{c}}kovi{\'c}, William Fedus, William~L Hamilton, Pietro Li{\`o},
  Yoshua Bengio, and R~Devon Hjelm.
\newblock Deep graph infomax.
\newblock {\em arXiv preprint arXiv:1809.10341}, 2018.

\bibitem{grover2016node2vec}
Aditya Grover and Jure Leskovec.
\newblock node2vec: Scalable feature learning for networks.
\newblock In {\em KDD}, pages 855--864, 2016.

\bibitem{tang2015line}
Jian Tang, Meng Qu, Mingzhe Wang, Ming Zhang, Jun Yan, and Qiaozhu Mei.
\newblock Line: Large-scale information network embedding.
\newblock In {\em WWW}, pages 1067--1077, 2015.

\bibitem{yang2020understanding}
Zhen Yang, Ming Ding, Chang Zhou, Hongxia Yang, Jingren Zhou, and Jie Tang.
\newblock Understanding negative sampling in graph representation learning.
\newblock {\em arXiv preprint arXiv:2005.09863}, 2020.

\bibitem{harris1954distributional}
Zellig~S Harris.
\newblock Distributional structure.
\newblock {\em Word}, 10(2-3):146--162, 1954.

\bibitem{gutmann2010noise}
Michael Gutmann and Aapo Hyv{\"a}rinen.
\newblock Noise-contrastive estimation: A new estimation principle for
  unnormalized statistical models.
\newblock In {\em AISTATS}, pages 297--304, 2010.

\bibitem{dijkstra1959note}
Edsger~W Dijkstra.
\newblock A note on two problems in connexion with graphs.
\newblock {\em Numerische mathematik}, 1(1):269--271, 1959.

\bibitem{715934}
M.~L. {Fredman} and R.~E. {Tarjan}.
\newblock Fibonacci heaps and their uses in improved network optimization
  algorithms.
\newblock In {\em 25th Annual Symposium on Foundations of Computer Science,
  1984.}, pages 338--346, 1984.

\bibitem{6690045}
L.~{Fu} and J.~{Deng}.
\newblock Graph calculus: Scalable shortest path analytics for large social
  graphs through core net.
\newblock In {\em WI-IAT}, volume~1, pages 417--424, 2013.

\bibitem{cohen2003reachability}
Edith Cohen, Eran Halperin, Haim Kaplan, and Uri Zwick.
\newblock Reachability and distance queries via 2-hop labels.
\newblock {\em SIAM Journal on Computing}, 32(5):1338--1355, 2003.

\bibitem{jin20093}
Ruoming Jin, Yang Xiang, Ning Ruan, and David Fuhry.
\newblock 3-hop: a high-compression indexing scheme for reachability query.
\newblock In {\em SIGMOD}, pages 813--826, 2009.

\bibitem{DBLP:journals/corr/ZitnikL17}
Marinka Zitnik and Jure Leskovec.
\newblock Predicting multicellular function through multi-layer tissue
  networks.
\newblock {\em CoRR}, abs/1707.04638, 2017.

\end{thebibliography}
}
\section*{Appendix A: Proofs of Theorems}

\text{THEOREM 4.1}
\textit{
Unigram Negative Sampling (UNS) Algorithm optimizes the generalized GRL objective with the following specifications:
    $\sm (i,j) = log(\sigma(\bz_i^T\bz_j))$, ~ $\dsm (i,j) = log(\sigma(-\bz_i^T\bz_j))$,
    $\alpha_d = \pi_d(C,\bA)$, where $\pi_d(C,\bA)$ is the probability of sampling a node-pair at distance $d$ using a $C$-length random walk on the graph with adjacency matrix $\bA$, and
    $\beta_d = KC/n.$
As a result, the Separation Power of UNS algorithm is equal to 1.
}
\vspace{-5ex}
\begin{proof}
To understand the relationship between the generalized GRL objective and UNS, let us look at the generic objective function of negative sampling (that applies for both UNS and DNS):
\vspace{-2ex}
\begin{align}
    E(\cZ) =& ~ \sum_{i\in\mathcal{V}} \sum_{j\in\mathcal{N}(i)}[\log\:\sigma(\bz_i^T \bz_j) + \nonumber\\ 
    &{K} \sum_{k\in \mathcal{V}} P_{neg}(k|i)\log(\sigma(- \bz_i^T \bz_k))] \nonumber
    \end{align}
where $\mathcal{N}(i)$ represents the set of nodes (of size $C$) that belong to the neighborhood of node $i$ which we construct using random walk strategy. If we denote the probability of sampling a node $j$ in a $C$-length random walk from $i$ as $P_{walk}(j|i)$, then $P_{walk}(j|i) = \pi_d(C,\bA)$, where $d$ is the distance between nodes $i$ and $j$. The objective function of generic negative sampling can thus be written as:
    \begin{align}
    E(\cZ) =& \sum_{i\in\mathcal{V}} \sum_{d=1}^{d_{max}} \sum_{j\in\mathcal{V} \land d(i,j) = d} \pi_{d}({C}, \bA) \log\:\sigma(\bz_i^T \bz_j) + \label{eqn:10a}\\
    &\sum_{i\in\mathcal{V}} \sum_{j\in\mathcal{N}(i)}{K} \sum_{k\in \mathcal{V}} P_{neg}(k|i)\log(\sigma(- \bz_i^T \bz_k))\nonumber\\
    =& ~ \sum_{d=1}^{d_{max}} \sum_{(i,j) \in \cS_d} \pi_{d}({C}, \bA)\log\:\sigma(\bz_i^T \bz_j) +  \nonumber\\
    &~~ {KC}\sum_{i\in\mathcal{V}}  \sum_{k\in \mathcal{V}} P_{neg}(k|i)\log(\sigma(- \bz_i^T \bz_k)) \nonumber\\
    =& ~ \sum_{d=1}^{d_{max}} \sum_{(i,j) \in \cS_d} \pi_{d}({C}, \bA)\log\:\sigma(\bz_i^T \bz_j) +  \nonumber\\
    &~~ {KC} \sum_{d=1}^{d_{max}} \sum_{(i,j) \in \cS_d} P_{neg}(k|i)\log(\sigma(- \bz_i^T \bz_k)) \nonumber
\end{align}
In UNS, $P_{neg}(j|i) = \frac{1}{n}$ where $n$ is the number of nodes. Hence, the objective function for UNS becomes:
\begin{align}
    E(\cZ) =& ~\sum_{d=1}^{d_{max}} \sum_{(i,j) \in \cS_d} \pi_{d}({C}, \bA)\log\:\sigma(\bz_i^T \bz_j) +  \nonumber \\
    &~~ \frac{KC}{n} \sum_{d=1}^{d_{max}} \sum_{(i,j) \in \cS_d} ~\log(\sigma(- \bz_i^T \bz_j))\nonumber\\
    =& ~\sum_{d=1}^{d_{max}} \sum_{(i,j) \in \cS_d} [\pi_{d}({C}, \bA)\log\:\sigma(\bz_i^T \bz_j) +  \frac{KC}{n} \log(\sigma(- \bz_i^T \bz_j)) ] \nonumber \\
    =& ~\sum_{d=1}^{d_{max}} \sum_{(i,j) \in \cS_d} [\alpha_d ~ \sm(i,j) + \beta_d ~ \dsm(i,j)] \nonumber
\end{align}
where, $\sm (i,j) = log(\sigma(\bz_i^T\bz_j))$, ~ $\dsm (i,j) = log(\sigma(-\bz_i^T\bz_j))$, $\alpha_d = \pi_d(C,\bA)$, and    $\beta_d = KC/n$. Clearly, the above equation corresponds to the generalized GRL objective.
The Separation Power of UNS is equal to $\frac{\beta_{d_{max}}}{\beta_{1}} = \frac{{K}{C}/n}{{K}{C}/n} = 1$.
\end{proof}

\text{THEOREM 4.2}
\textit{
Distance-aware Negative Sampling (DNS) Algorithm optimizes the generalized GRL objective with the following specifications:
    $\sm (i,j) = log(\sigma(\bz_i^T\bz_j))$, ~ $\dsm (i,j) = log(\sigma(-\bz_i^T\bz_j))$,
    $\alpha_d = \pi_d(C,\bA)$, where $\pi_d(C,\bA)$ is the probability of sampling a node-pair at distance $d$ using a $C$-length random walk on the graph with adjacency matrix $\bA$, and
    $\beta_d = KCd/\mathcal{D}(\bA).$
As a result, the Separation Power of DNS algorithm is equal to $d_{max}$.
}
\begin{proof}
For Distance-aware Negative Sampler (DNS), the negative sampling probability $P_{neg}(j|i)$ is linearly proportional to the pair-wise distance $d(j, i)$; $P_{neg}(j|i) = \frac{d(j, i)}{\cD(i, \bA)} = \frac{d}{\cD(i, \bA)}$. We approximate the expected value of $\cD(i, \bA)$ as $\cD(\bA)$. The objective function for DNS can thus be obtained by substituting the value of $P_{neg}(j|i)$ in Equation \ref{eqn:10a} as follows:

\begin{align}
    E(\cZ) =& ~\sum_{d=1}^{d_{max}} \sum_{(i,j) \in \cS_d} \pi_{d}({C}, \bA)\log\:\sigma(\bz_i^T \bz_j) +  \label{eqn:13}\\ &~~\sum_{d=1}^{d_{max}} \sum_{(i,j) \in \cS_d} {K}{C}~\frac{d}{\cD(\bA)}~\log(\sigma(- \bz_i^T \bz_j))\nonumber\\
        =& ~\sum_{d=1}^{d_{max}} \sum_{(i,j) \in \cS_d} [\pi_{d}({C}, \bA)\log\:\sigma(\bz_i^T \bz_j) +  \nonumber\\ &\frac{KCd}{\cD(\bA)} \log(\sigma(- \bz_i^T \bz_j)) ] \nonumber\\
    =& ~\sum_{d=1}^{d_{max}} \sum_{(i,j) \in \cS_d} [\alpha_d ~ \sm(i,j) + \beta_d ~ \dsm(i,j)] \nonumber
\end{align}
where, $\sm (i,j) = log(\sigma(\bz_i^T\bz_j))$, ~ $\dsm (i,j) = log(\sigma(-\bz_i^T\bz_j))$, $\alpha_d = \pi_d(C,\bA)$, and    $\beta_d = KCd/\mathcal{D}(\bA)$. Clearly, Equation \ref{eqn:13} corresponds to the generalized GRL objective.
The Separation Power of DNS is equal to $\frac{\beta_{d_{max}}}{\beta_{1}} = \frac{{K}{C}d_{max}/\cD(\bA)}{{K}{C}/\cD(\bA)} = d_{max}$.
\end{proof}

\text{THEOREM 4.3}
\textit{
Let the average pairwise similarity for any two nodes at distance d be given by $\xi_d$ = $\frac{1}{|\cS_{d}|}\sm(i, j) = \frac{1}{|\cS_{d}|} \sum_{(i,j)\in \cS_d} \sigma (z_i^Tz_j)$.  
We can then show that DNS generates embeddings such that $\xi_d$ is a function of $d$ and for $d > C$, $\xi_{d}$ is inversely proportional to $d$. 
}
\vspace{-1ex}
\begin{proof}
Let us assume that the DNS based GRL model has reached its global maximum with loss $\mathcal{Q}$. From Equation \ref{eqn:13}, the loss of DNS based GRL model is given by,
\vspace{-2ex}
\begin{align}
    \cQ &= -~\sum_{d=1}^{d_{max}} \sum_{(i,j) \in \cS_d} \pi_{d}({C}, \bA)\log\:\sigma(\bz_i^T \bz_j) ~- \label{eqn:15}\\ &~~{K}{C}\sum_{d=1}^{d_{max}} \sum_{(i,j) \in \cS_d} \frac{d}{\cD(\bA)}~\log(\sigma(- \bz_i^T \bz_j))\nonumber\\
    &= -~\sum_{d=1}^{d_{max}} \sum_{(i,j) \in \cS_d} \pi_{d}({C}, \bA)\log\:\sigma(\bz_i^T \bz_j) ~- \nonumber\\
    &~~{K}{C}\sum_{d=1}^{d_{max}} \sum_{(i,j) \in \cS_d} \frac{d}{\cD(\bA)}~\log(1 -\sigma(\bz_i^T \bz_j))\nonumber
\end{align}
Since the model has reached its global optimum, the similarity $\sigma(\bz_i^T\bz_j)$ for nearby node-pairs ($d(i, j)<C$) should be high such that $1 - \sigma(\bz_i^T\bz_j)$ low. We approximate $log(\sigma(\bz_i^T\bz_j)) \approx \sigma(\bz_i^T\bz_j) - 1$ as the remainder term in its Taylor's series expansion is close to zero. Moreover, we expand $log(1 - \sigma(\bz_i^T\bz_j)) = - \sigma(\bz_i^T\bz_j) - \frac{\sigma(\bz_i^T\bz_j)^2}{2}- \frac{\sigma(\bz_i^T\bz_j)^3}{3} - \dots = - \sigma(\bz_i^T\bz_j) - \mathcal{R}(\sigma(\bz_i^T\bz_j))$.
Note that the length of the random walk is at most ${C}$. Consequently, for $d > {C}$, $\pi_{d}({C}, \bA) = 0$. We can thus rearrange Equation \ref{eqn:15} as,
\vspace{-2ex}
\begin{align*}
    \cQ &= -~\sum_{d=1}^{C} \sum_{(i,j) \in \cS_d} \pi_{d}({C}, \bA)(\sigma(\bz_i^T\bz_j) - 1) ~\\ \nonumber
    &\phantom{{}={}} -{K}{C}\sum_{d=1}^{d_{max}} \sum_{(i,j) \in \cS_d} \frac{d}{\cD(\bA)}(- \sigma(\bz_i^T\bz_j) - \mathcal{R}(\sigma(\bz_i^T\bz_j)))\\ \nonumber
    &= -~\sum_{d=1}^{C} \sum_{(i,j) \in \cS_d} \pi_{d}({C}, \bA)(\sigma(\bz_i^T\bz_j) - 1) ~\\ \nonumber
    &\phantom{{}={}} +{K}{C}\sum_{d=1}^{d_{max}} \sum_{(i,j) \in \cS_d} (\frac{d}{\cD(\bA)}\sigma(\bz_i^T\bz_j) + \frac{d}{\cD(\bA)}\mathcal{R}(\sigma(\bz_i^T\bz_j)))\\ \nonumber
\end{align*}
We approximate $\frac{d}{\cD(\bA)}\mathcal{R}(\sigma(\bz_i^T\bz_j)) \approx 0$ because $R(\sigma(\bz_i^T\bz_j)) \approx 0$ for distant pairs and $\frac{d}{\cD(\bA)} \approx 0$ for nearby pairs at optimum. Hence, we obtain:
\vspace{-2ex}
\begin{align}
    \cQ &= -~\sum_{d=1}^{C} \sum_{(i,j) \in \cS_d} \pi_{d}({C}, \bA)\sigma(\bz_i^T\bz_j) + \label{eqn:17} \\
    &~\sum_{d=1}^{C} \sum_{(i,j) \in \cS_d} \pi_{d}({C}, \bA) + {K}{C}\sum_{d=1}^{d_{max}} \sum_{(i,j) \in \cS_d} \frac{d}{\cD(\bA)}\sigma(\bz_i^T\bz_j) \nonumber \\ 
    &= ~\sum_{d=1}^{C} \sum_{(i,j) \in \cS_d} \pi_{d}({C}, \bA) -~\sum_{d=1}^{C} \sum_{(i,j) \in \cS_d} (\pi_{d}({C}, \bA)- \frac{{K}{C}d}{\cD(\bA)}) \nonumber\\ 
    &\phantom{{}={}} \sigma(\bz_i^T\bz_j) +\sum_{d={C}+1}^{d_{max}} ~\sum_{(i,j) \in \cS_d} \frac{{K}{C}d}{\cD(\bA)}\sigma(\bz_i^T\bz_j) \nonumber
\end{align}
\vspace{-3.75ex}
For $d \leq {C}$, we rearrange Equation \ref{eqn:17}  as,
\begin{align*}
    \cQ &= ~\sum_{d'=1}^{C} ~\sum_{(i,j) \in \cS_{d'}} \pi_{d'}({C}, \bA) -\nonumber\\
    &~\sum_{d'=1}^{C} \sum_{(i,j) \in \cS_{d'}} (\pi_{d'}({C}, \bA)- \frac{{K}{C}d'}{\cD(\bA)})~\sigma(\bz_i^T\bz_j) \\ \nonumber
    &\phantom{{}={}} +\sum_{d'={C}+1}^{d_{max}} ~\sum_{(i,j) \in \cS_{d'}} \frac{{K}{C}d'}{\cD(\bA)}\sigma(\bz_i^T\bz_j) \\ \nonumber
    &= ~\sum_{d'=1}^{C} ~\sum_{(i,j) \in \cS_{d'}} \pi_{d'}({C}, \bA)-\nonumber \\
    &\sum_{(i,j) \in \cS_{d}} (\pi_{d}({C}, \bA)- \frac{{K}{C}d}{\cD(\bA)})~\sigma(\bz_i^T\bz_j) \\ \nonumber
    &\phantom{{}={}} -~\sum_{\substack{d'=1 \\ d' \neq d}}^{C}~ \sum_{(i,j) \in \cS_{d'}} (\pi_{d'}({C}, \bA)- \frac{{K}{C}d'}{\cD(\bA)})~\sigma(\bz_i^T\bz_j) +\nonumber\\
    &~~\sum_{d'={C}+1}^{d_{max}} ~\sum_{(i,j) \in \cS_{d'}} \frac{{K}{C}d'}{\cD(\bA)}\sigma(\bz_i^T\bz_j) \\ \nonumber
    &= ~\sum_{d'=1}^{C} \sum_{(i,j) \in \cS_{d'}} \pi_{d'}({C}, \bA)- (\pi_{d}({C}, \bA)- \frac{{K}{C}d}{\cD(\bA)})|\cS_d|\xi_d \\ \nonumber
    &\phantom{{}={}} -~\sum_{\substack{d'=1 \\ d' \neq d}}^{C}~ \sum_{(i,j) \in \cS_{d'}} (\pi_{d'}({C}, \bA)-\frac{{K}{C}d'}{\cD(\bA)})~\sigma(\bz_i^T\bz_j) +\nonumber\\
    &~~\sum_{d'={C}+1}^{d_{max}} ~\sum_{(i,j) \in \cS_{d'}} \frac{{K}{C}d'}{\cD(\bA)}\sigma(\bz_i^T\bz_j) \nonumber
\end{align*}
\begin{align*}
    \xi_d &= ~\frac{1}{|\cS_d|~(\pi_{d}({C}, \bA)- \frac{{K}{C}d}{\cD(\bA)})} \times \nonumber\\
    &[-\cQ + ~~~\sum_{d'=1}^{C} ~\sum_{(i,j) \in \cS_{d'}} \pi_{d'}({C}, \bA) \\ \nonumber
    &\phantom{{}={}} -~\sum_{\substack{d'=1 \\ d' \neq d}}^{C}~ \sum_{(i,j) \in \cS_{d'}} (\pi_{d'}({C}, \bA)- \frac{{K}{C}d'}{\cD(\bA)})~\sigma(\bz_i^T\bz_j) +\nonumber\\
    &~~\sum_{d'={C}+1}^{d_{max}} ~\sum_{(i,j) \in \cS_{d'}} \frac{{K}{C}d'}{\cD(\bA)}\sigma(\bz_i^T\bz_j)] \nonumber
\end{align*}
From the above Equation, $\xi_d = f(d, \Theta)$ for $d \leq {C}$, where $\Theta$ is the set of parameters of $f$ other than $d$.

For $d > {C}$, we rearrange Equation \ref{eqn:17}  as,
\begin{align*}
    \cQ &= ~\sum_{d'=1}^{C} ~\sum_{(i,j) \in \cS_{d'}} \pi_{d'}({C}, \bA) -~\sum_{d'=1}^{C}\sum_{(i,j) \in \cS_{d'}} (\pi_{d'}({C}, \bA) \nonumber\\
    &-\frac{{K}{C}d'}{\cD(\bA)})~\sigma(\bz_i^T\bz_j) +\sum_{d'={C}+1}^{d_{max}} ~\sum_{(i,j) \in \cS_{d'}} \frac{{K}{C}d'}{\cD(\bA)}\sigma(\bz_i^T\bz_j) \\ \nonumber
    &= ~\sum_{d'=1}^{C} ~\sum_{(i,j) \in \cS_{d'}} \pi_{d'}({C}, \bA) -~\sum_{d'=1}^{C}~ \sum_{(i,j) \in \cS_{d'}} (\pi_{d'}({C}, \bA) \nonumber\\
    &-\frac{{K}{C}d'}{\cD(\bA)})~\sigma(\bz_i^T\bz_j) +~\sum_{(i,j) \in \cS_{d}} \frac{{K}{C}d}{\cD(\bA)}\sigma(\bz_i^T\bz_j) +\nonumber\\
    &~~\sum_{\substack{d'={C}+1 \\ d' \neq d}}^{d_{max}} ~\sum_{(i,j) \in \cS_{d'}} \frac{{K}{C}d'}{\cD(\bA)}\sigma(\bz_i^T\bz_j) \\ \nonumber
     &= ~\sum_{d'=1}^{C} ~\sum_{(i,j) \in \cS_{d'}} \pi_{d'}({C}, \bA) -~\sum_{d'=1}^{C}~ \sum_{(i,j) \in \cS_{d'}} (\pi_{d'}({C}, \bA) \nonumber\\
     &-\frac{{K}{C}d'}{\cD(\bA)})~\sigma(\bz_i^T\bz_j) +\frac{{K}{C}d}{\cD(\bA)}~|\cS_d|\xi_d \nonumber\\
     &+~~\sum_{\substack{d'={C}+1 \\ d' \neq d}}^{d_{max}} ~\sum_{(i,j) \in \cS_{d'}} \frac{{K}{C}d'}{\cD(\bA)}\sigma(\bz_i^T\bz_j) \nonumber
\end{align*}
\begin{align*}
    \xi_d &= \frac{\cD(\bA)}{|\cS_d|{K}{C}~d} ~~[~\cQ -~\sum_{d'=1}^{C} ~\sum_{(i,j) \in \cS_{d'}} \pi_{d'}({C}, \bA)\\ \nonumber
    &\phantom{{}={}} +~\sum_{d'=1}^{C}~ \sum_{(i,j) \in \cS_{d'}} (\pi_{d'}({C}, \bA)- \frac{{K}{C}d'}{\cD(\bA)})~\sigma(\bz_i^T\bz_j) ~ \nonumber\\
    &-~\sum_{\substack{d'={C}+1 \\ d' \neq d}}^{d_{max}} ~\sum_{(i,j) \in \cS_{d'}} \frac{{K}{C}d'}{\cD(\bA)}\sigma(\bz_i^T\bz_j)~] \nonumber
\end{align*}
From the above Equation, $\xi_d$ is inversely proportional to $d$ for $d>{C}$.
\end{proof}

\section*{Appendix B: Implementation Details}

\paragraph{Synthetic Data Generation:}
We construct the synthetic networks by generating a node degree sequence that follows the power-law distribution. For our experiments, we use \textit{networkx.utils.powerlaw\_sequence} to generate the degree sequence which takes two parameters: the number of nodes and the exponent of the power-law distribution, where we set the number of nodes as 2,000 and vary the exponent to generate varying networks with different density. After that, we use \textit{networkx.expected\_degree\_graph} to construct a network from each degree sequence; whereas each network may have many disconnected components. To connect all the components of the network, we randomly choose one node from each disconnected component and connect them using minimum number of artificial edges.
We generate structure-induced node labels using a simple label propagation approach. Initially, we randomly select $k$ seed nodes for $k$ distinct classes. For sparse network, we choose $k=7$ classes, whereas, for moderate and dense networks, we choose $k=5$ and $k=4$ classes respectively. At each iteration, we propagate the node label to its adjacent unlabeled nodes. Consequently, we iterate this procedure until all the nodes get labeled. Therefore, the node labels are generated only using the structure information, such as the proximity from the seed node. 

\paragraph{Node Classification Setup:}
We use Logistic Regression (LR) with an lbfgs solver that supports 150 max iterations as our downstream model. For the PPI dataset, we use multi-class settings of LR.
For benchmark datasets, we use the PyTorch Geometric 
\textit{train-test-validation} mask on the largest component to generate the training nodes, testing nodes and validation nodes. Meanwhile, for the synthetic datasets, we randomly select 10\% nodes for training, 40\% for validation, and 40\% for testing.     
We ran our experiments in a single machine with 2 NVIDIA Titan RTX GPUs (24Gb of RAM) and  1 Intel(R) Xeon(R) W-2135 CPU (@ 3.70GHz). We use PyTorch with cuda-10.1 for our experiments.

\paragraph{Hyper-parameters:}
We set the embedding dimension as 128, the number of random walks per node as 50, and the number of negative samples as 20, whereas for node2vec model, we also set return parameter $p$ as 1 and inout parameter $q$ as 4 for our experiments. To optimize these models, we use Adam optimizer with 0.01 learning rate. Moreover, we run all node wise negative sampling-based GRL model for 30 epochs and all edgewise negative sampling base models for 400 epochs, as edgewise sampler models take more iterations to converge. 

\begin{figure*}[t]
\begin{subfigure}{.33\textwidth}
  \centering
  \includegraphics[width=.99\linewidth]{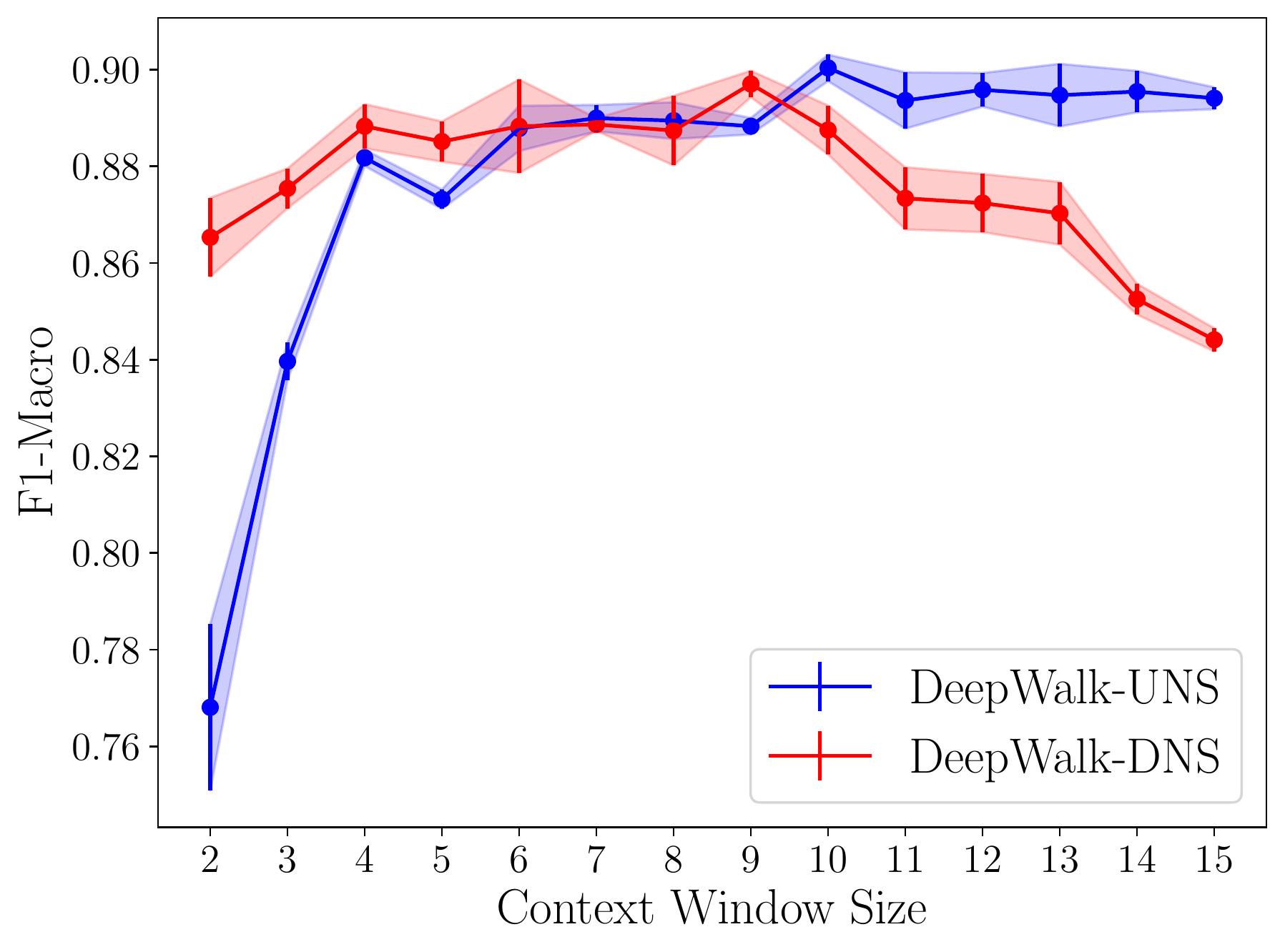}
  \caption*{Synthetic Sparse}
\end{subfigure}
\begin{subfigure}{.33\textwidth}
  \centering
  \includegraphics[width=.99\linewidth]{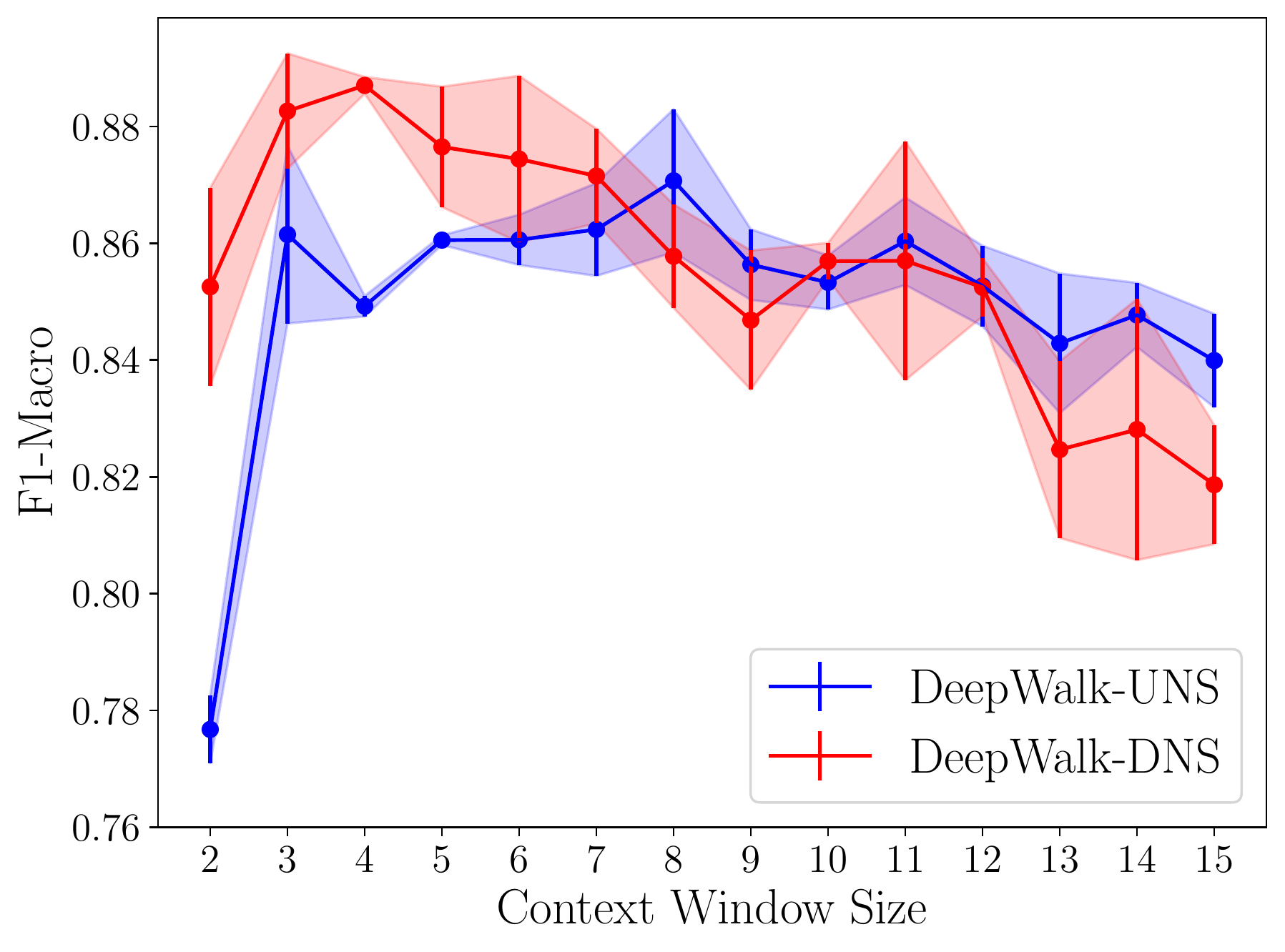}
  \caption*{Synthetic Moderate}
\end{subfigure}
\begin{subfigure}{.33\textwidth}
  \centering
  \includegraphics[width=.99\linewidth]{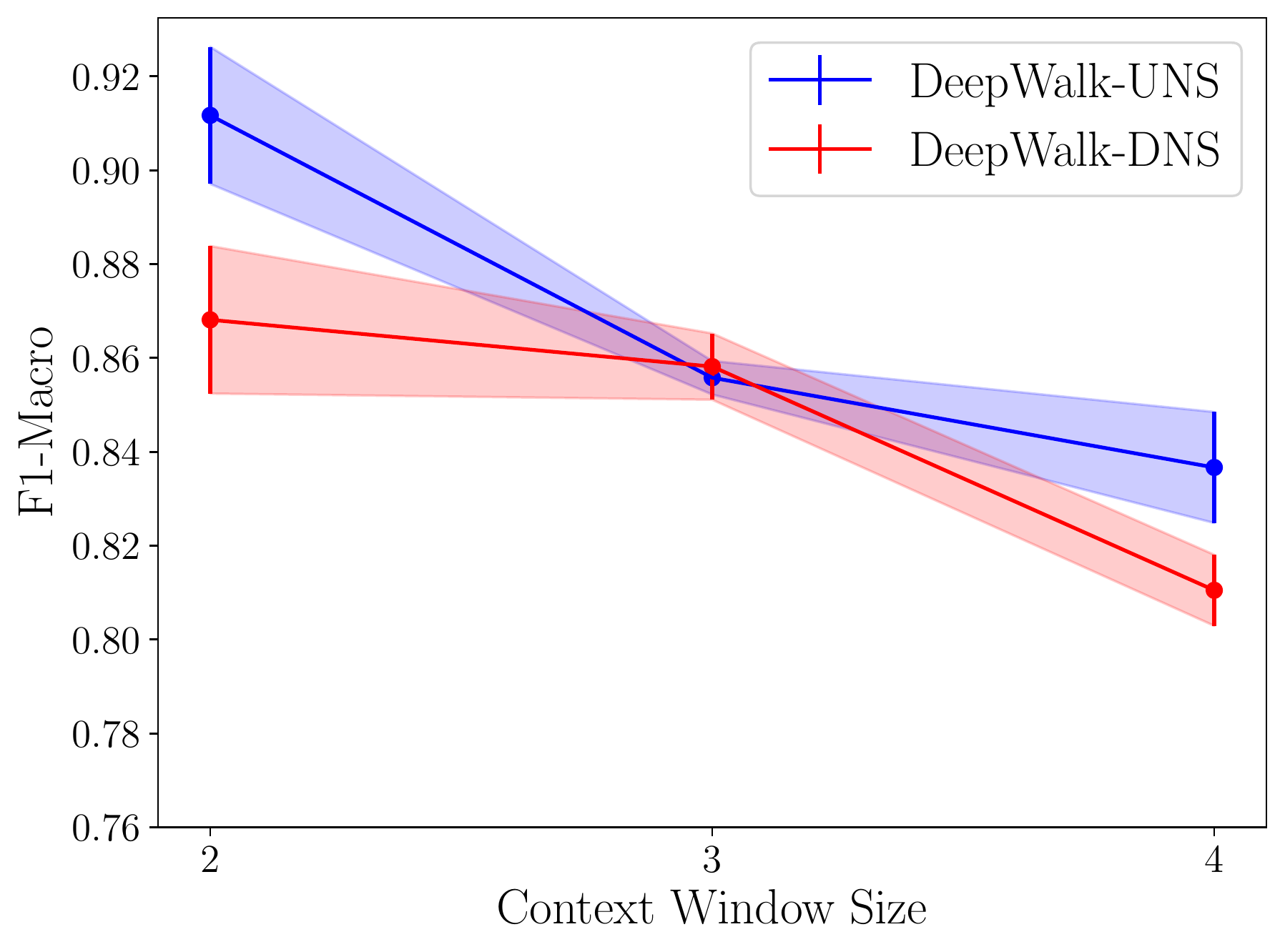}
  \caption*{Synthetic Dense}
\end{subfigure}

\bigskip
\begin{subfigure}{.33\textwidth}
  \centering
  \includegraphics[width=.99\linewidth]{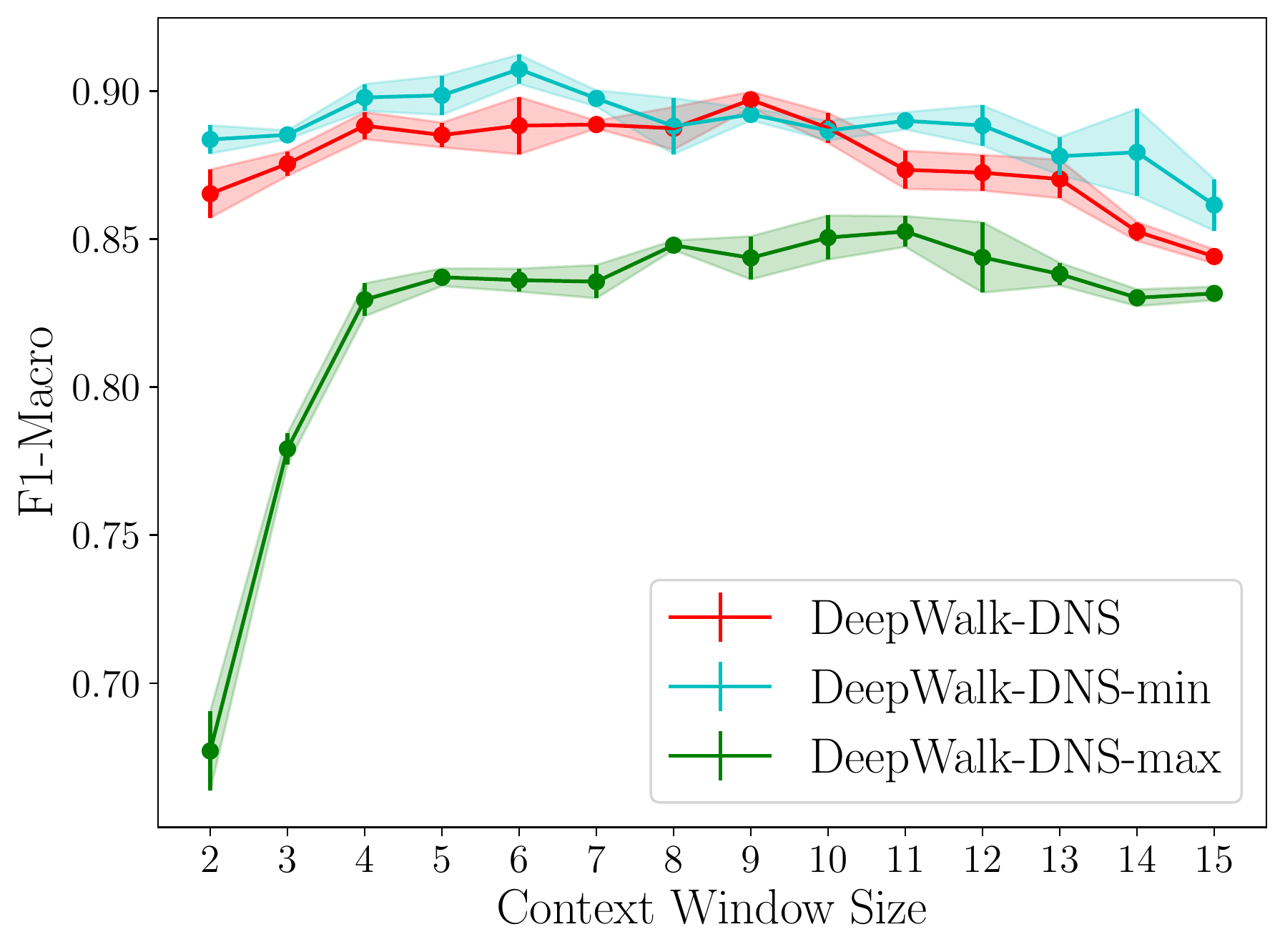}
  \caption*{Synthetic Sparse}
\end{subfigure}
\begin{subfigure}{.33\textwidth}
  \centering
  \includegraphics[width=.99\linewidth]{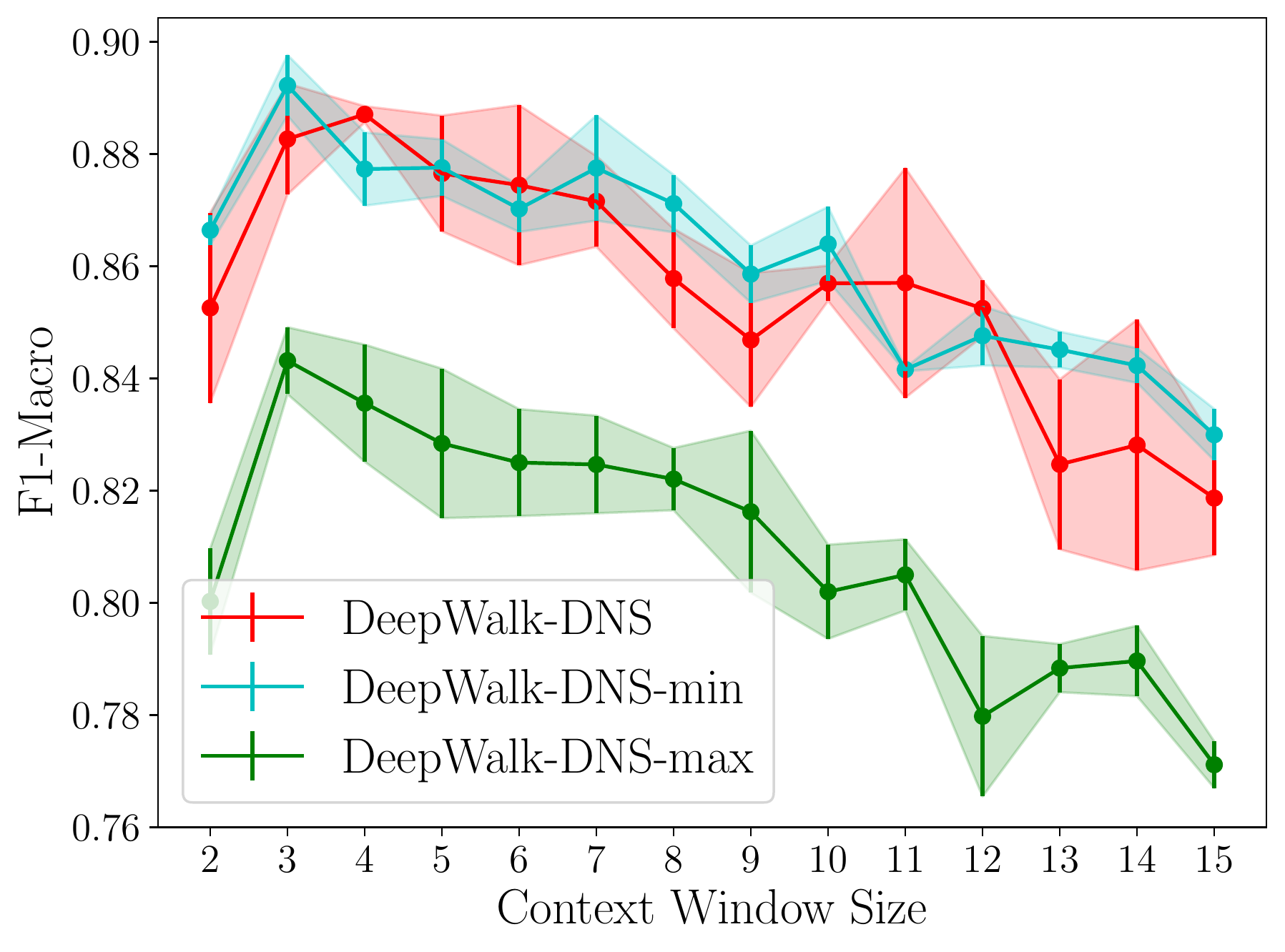}
  \caption*{Synthetic Moderate}
\end{subfigure}
\begin{subfigure}{.33\textwidth}
  \centering
  \includegraphics[width=.99\linewidth]{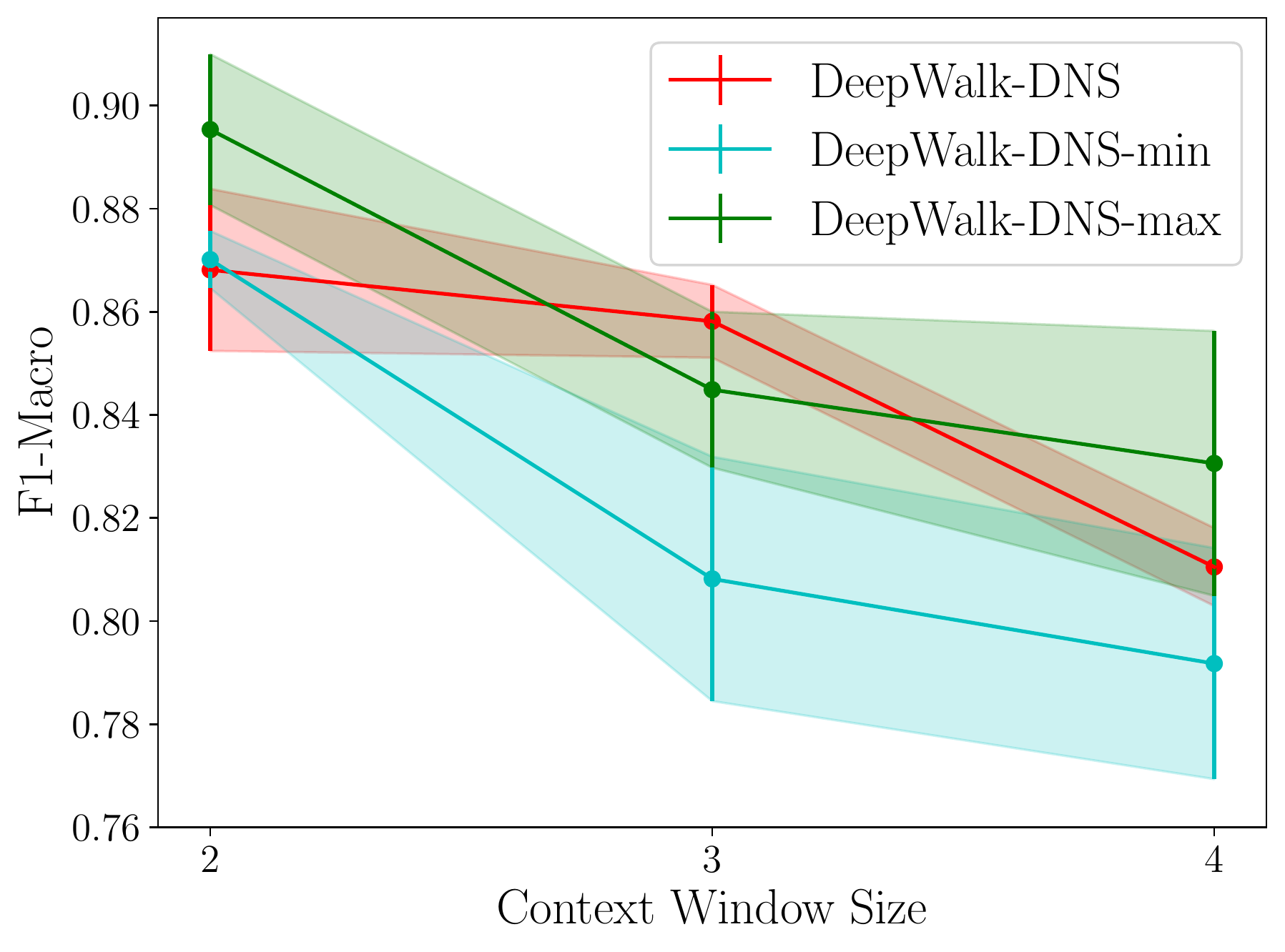}
  \caption*{Synthetic Dense}
\end{subfigure}

\caption{
F1-Macro score plot with varying context window on Synthetic Sparse, Synthetic Moderate, and Synthetic Dense dataset. Competing models are DeepWalk-UNS, DeepWalk-DNS, and its variants DeepWalk-DNS-min, DeepWalk-DNS-max.}
\label{fig:ablation}
\vspace{-3ex}

\end{figure*}

\section*{Appendix C: Additional Analysis of Results}

\begin{figure}[t]
    \centering
    \includegraphics[width=7cm,height=8cm,keepaspectratio]{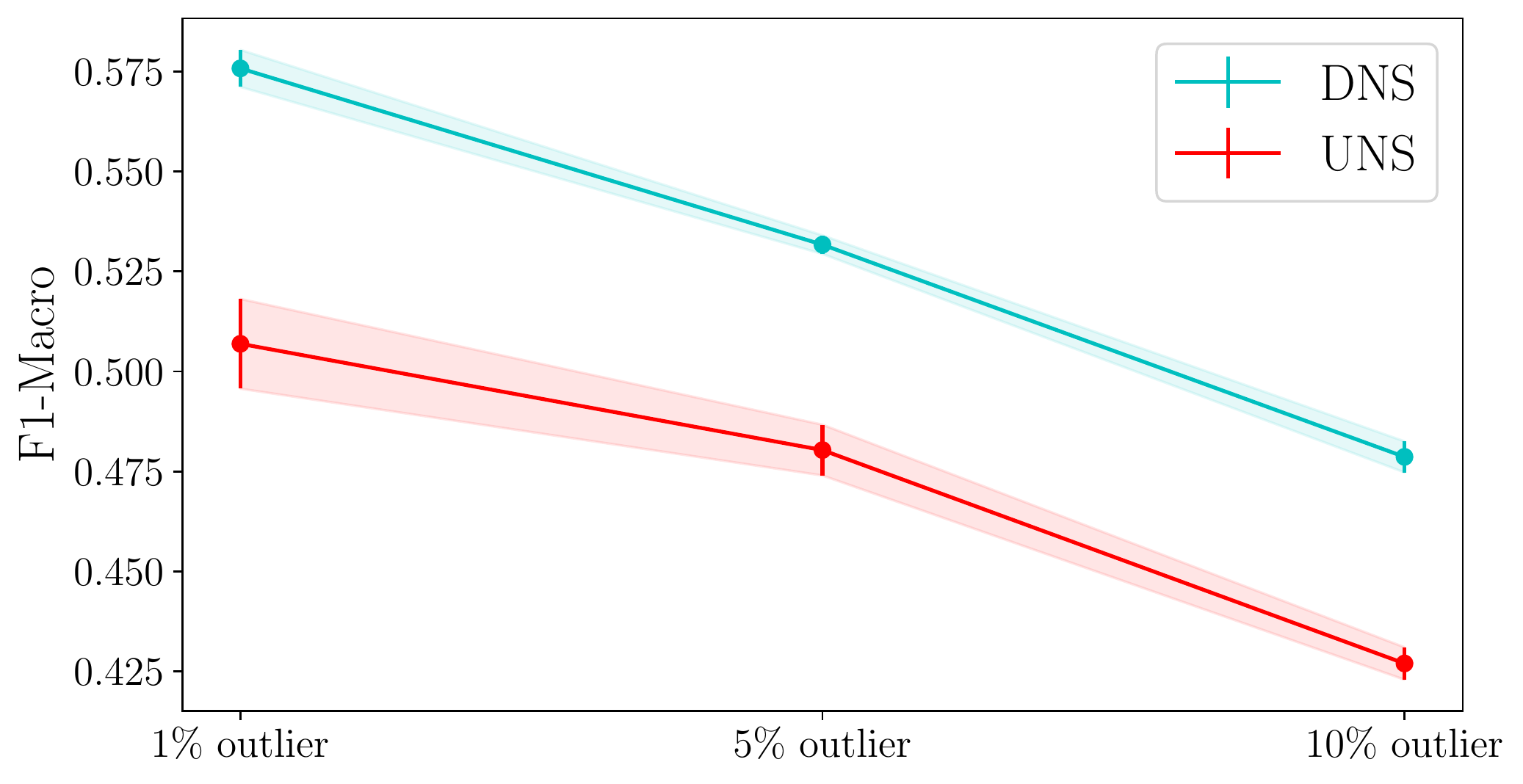}
    \caption{Sensitivity towards outlier points on CiteSeer dataset.}
    \label{fig:outlier}
    \vspace{-1ex}
\end{figure}

\subsubsection*{Ablation Study:}
We perform an ablation study of our DNS sampler by splitting the negative sampling probability into two parts; the splitting point is the pairwise distance for which DNS probability $\approx$ UNS probability. In the first ablation model, we set negative sampling probability linearly proportional to the pairwise distance for nearby nodes while maintaining uniform negative sampling probability for the rest of the nodes. In the second ablation model, we set a uniform negative sampling probability for nearby nodes while setting DNS-like probability for distant nodes. Let us denote the first sampler as DNS-min since its negative sampling probability $P_{min}(k|i) = \min (P_{DNS}(k|i), \frac{1}{n})$ and the second sampler as DNS-max that has negative sampling probability $P_{max}(k|i) = \max (P_{DNS}(k|i), \frac{1}{n})$. Both DNS-min and DNS-max samplers have higher separation than UNS sampler. 

Figure \ref{fig:ablation} shows the node classification performance of different samplers with the DeepWalk model. The top row of Figure \ref{fig:ablation} shows the F1-Macro score of DeepWalk-UNS and DeepWalk-DNS with varying context size. As discussed in Section 4, we see the performance of the DeepWalk-DNS model decreases with increasing context window size. Moreover, low negative sampling probability for nearby nodes is not effective for the synthetic dense graph. 

In the second row of Figure \ref{fig:ablation}, we see the comparison of DNS, DNS-min, and DNS-max in terms of node classification performance. The DeepWalk-DNS-max follows the trend of the DeepWalk-UNS model performance for lower context windows, whereas, the DeepWalk-DNS-min model more likely follows the trend of the DeepWalk-DNS model in all the synthetic graphs. 


\begin{figure*}
\begin{subfigure}{.33\textwidth}
  \centering
  \includegraphics[width=.99\linewidth]{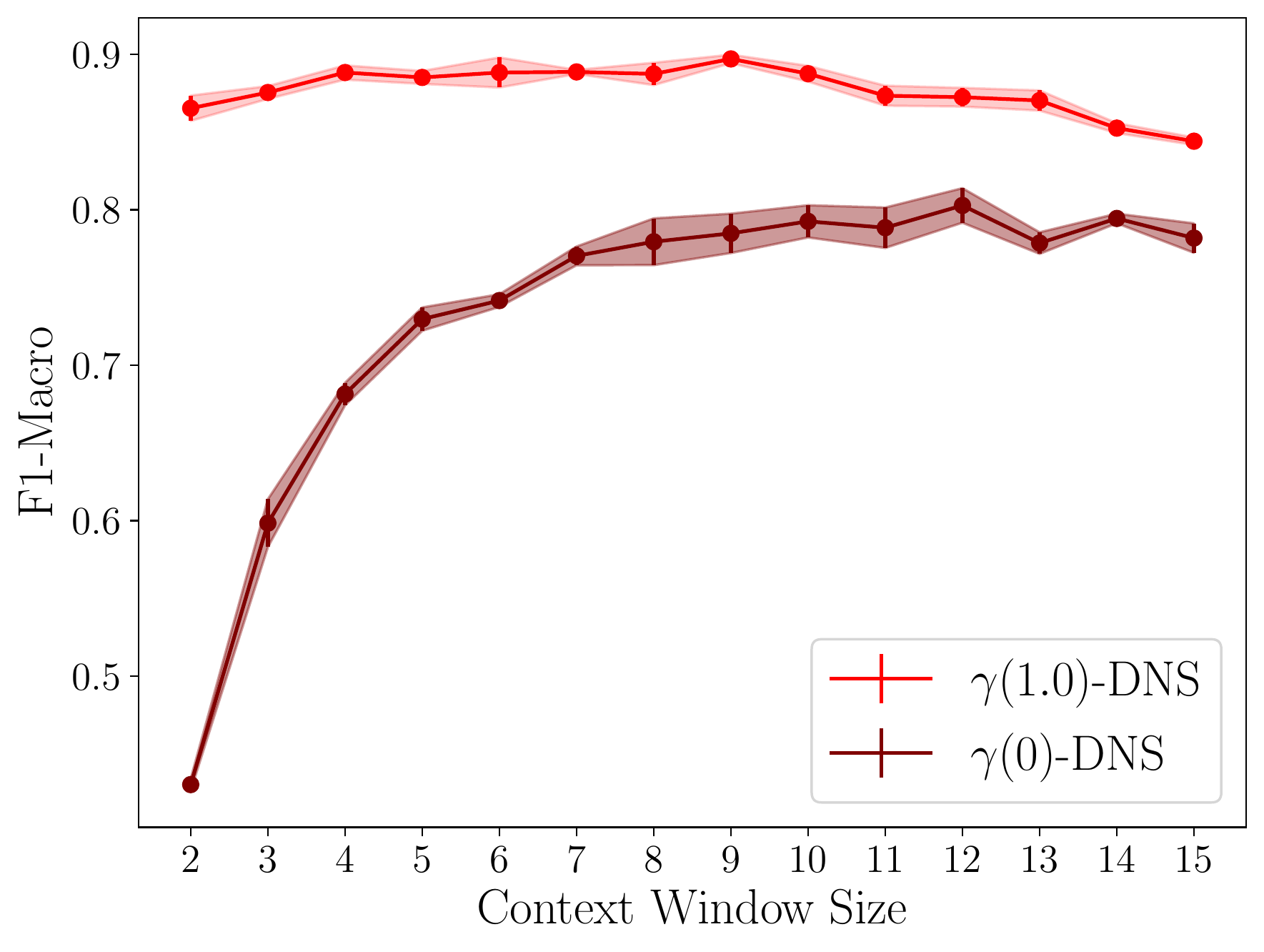}
\end{subfigure}
\begin{subfigure}{.33\textwidth}
  \centering
  \includegraphics[width=.99\linewidth]{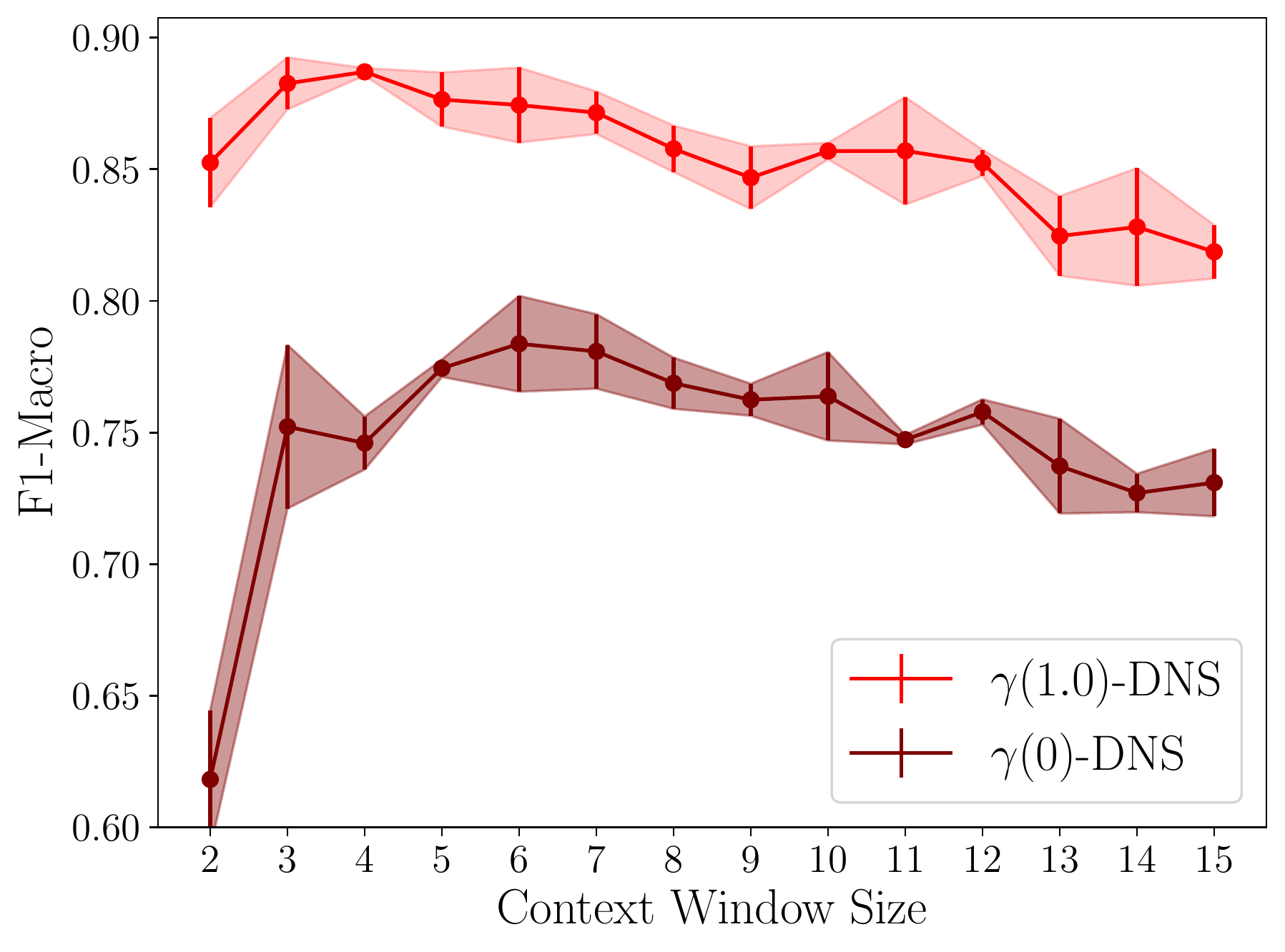}
\end{subfigure}
\begin{subfigure}{.33\textwidth}
  \centering
  \includegraphics[width=.99\linewidth]{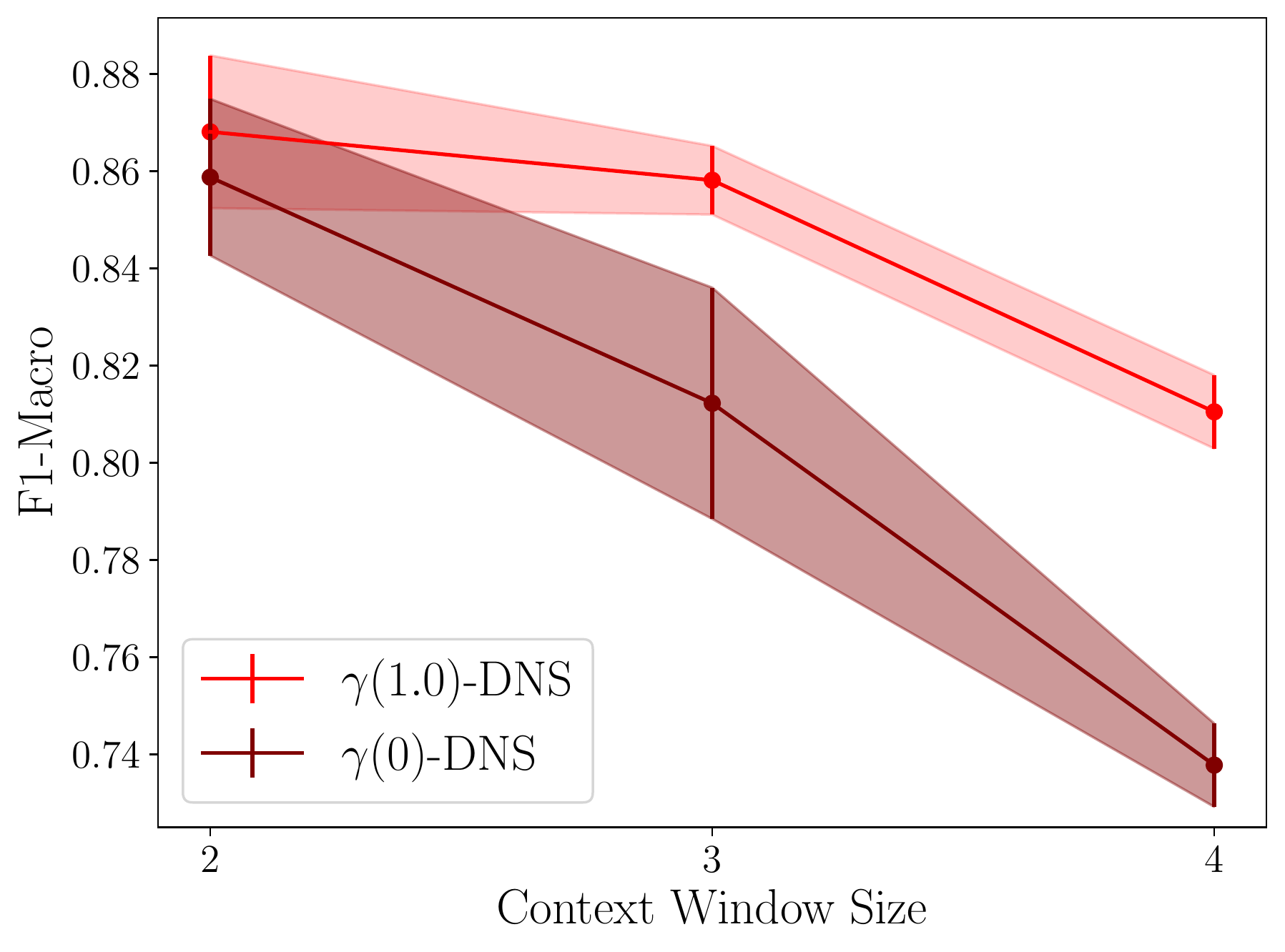}
\end{subfigure}

\bigskip
\begin{subfigure}{.33\textwidth}
  \centering
  \includegraphics[width=.99\linewidth]{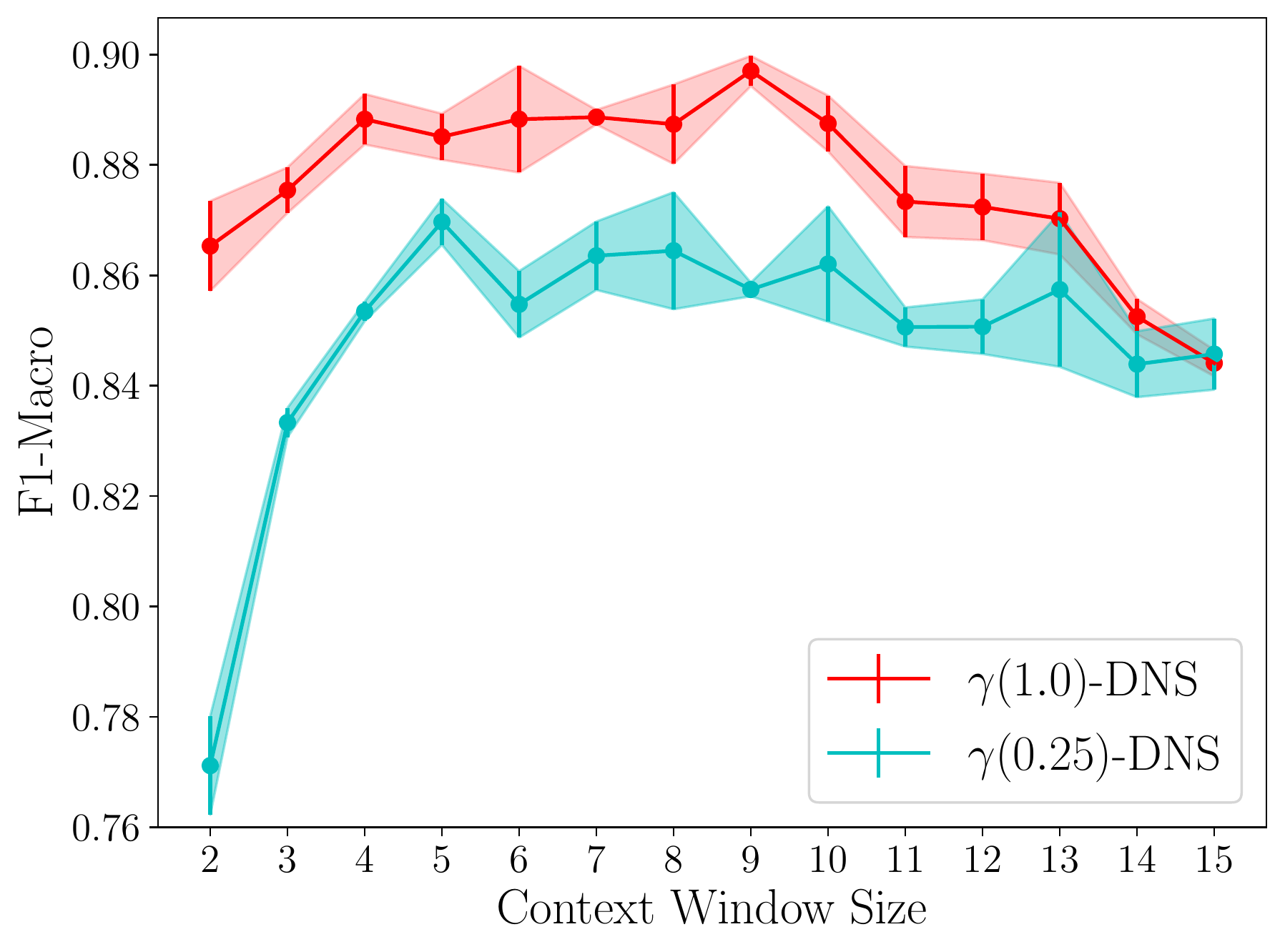}
\end{subfigure}
\begin{subfigure}{.33\textwidth}
  \centering
  \includegraphics[width=.99\linewidth]{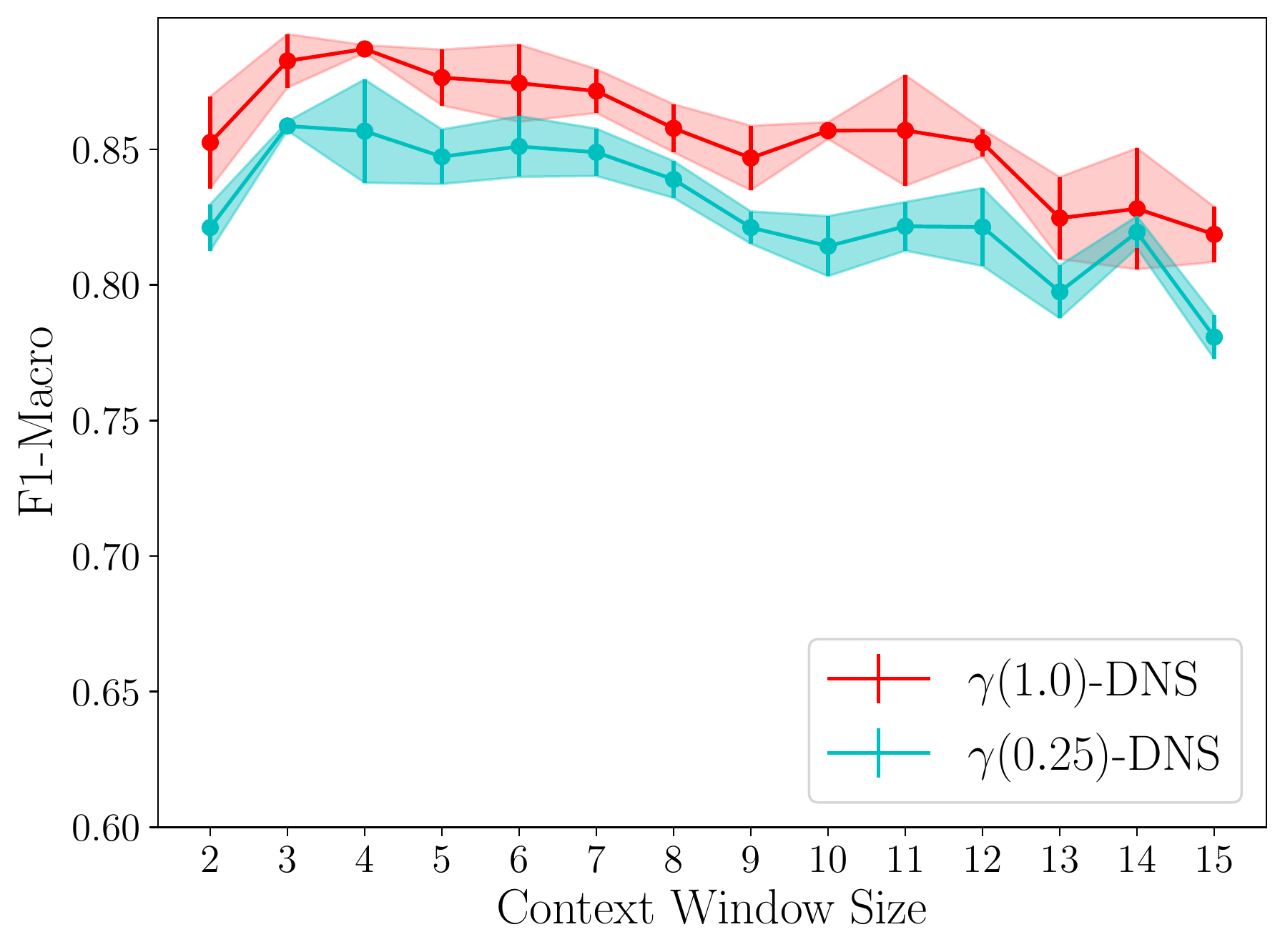}
\end{subfigure}
\begin{subfigure}{.33\textwidth}
  \centering
  \includegraphics[width=.99\linewidth]{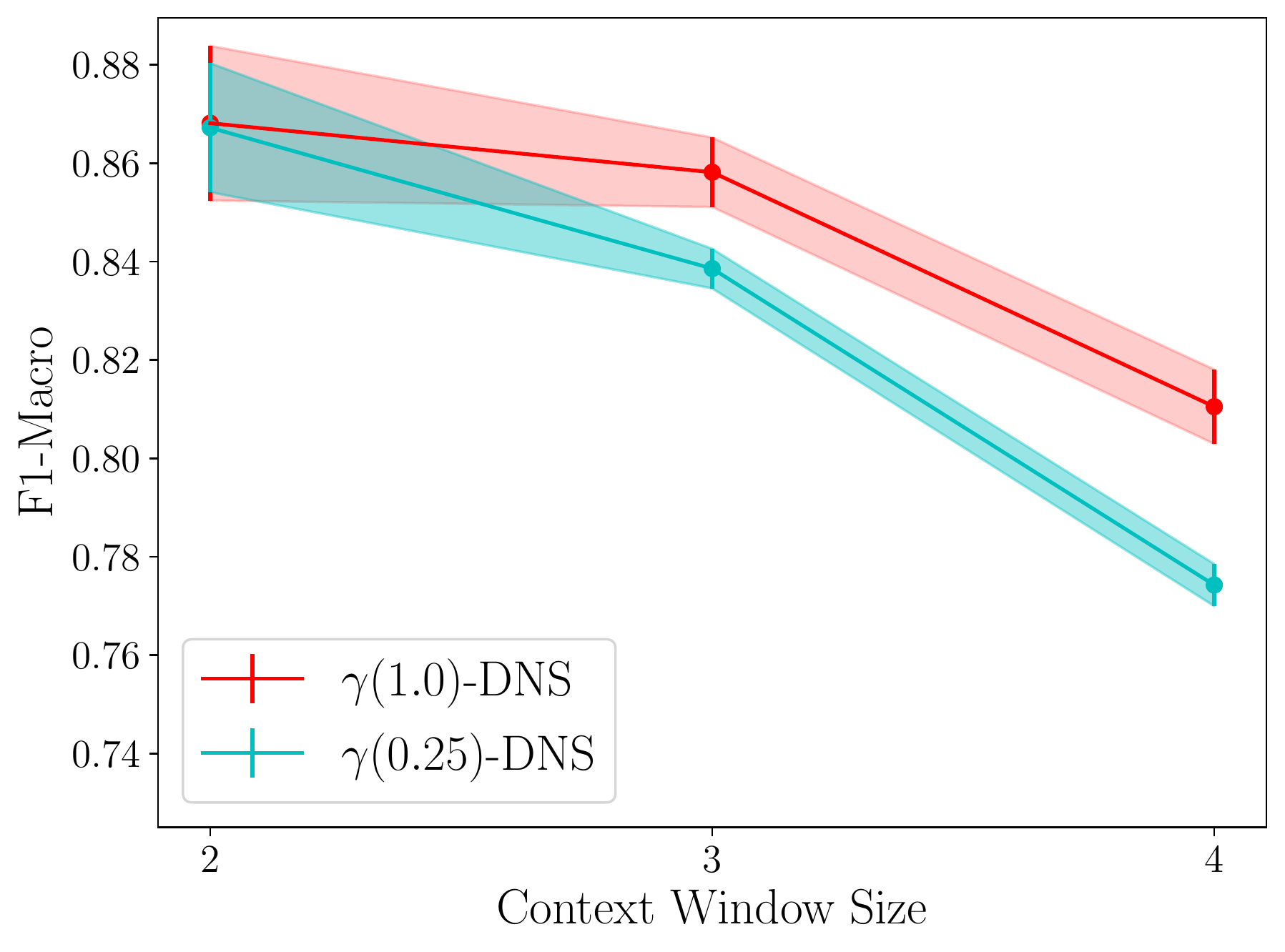}
\end{subfigure}

\bigskip
\begin{subfigure}{.33\textwidth}
  \centering
  \includegraphics[width=.99\linewidth]{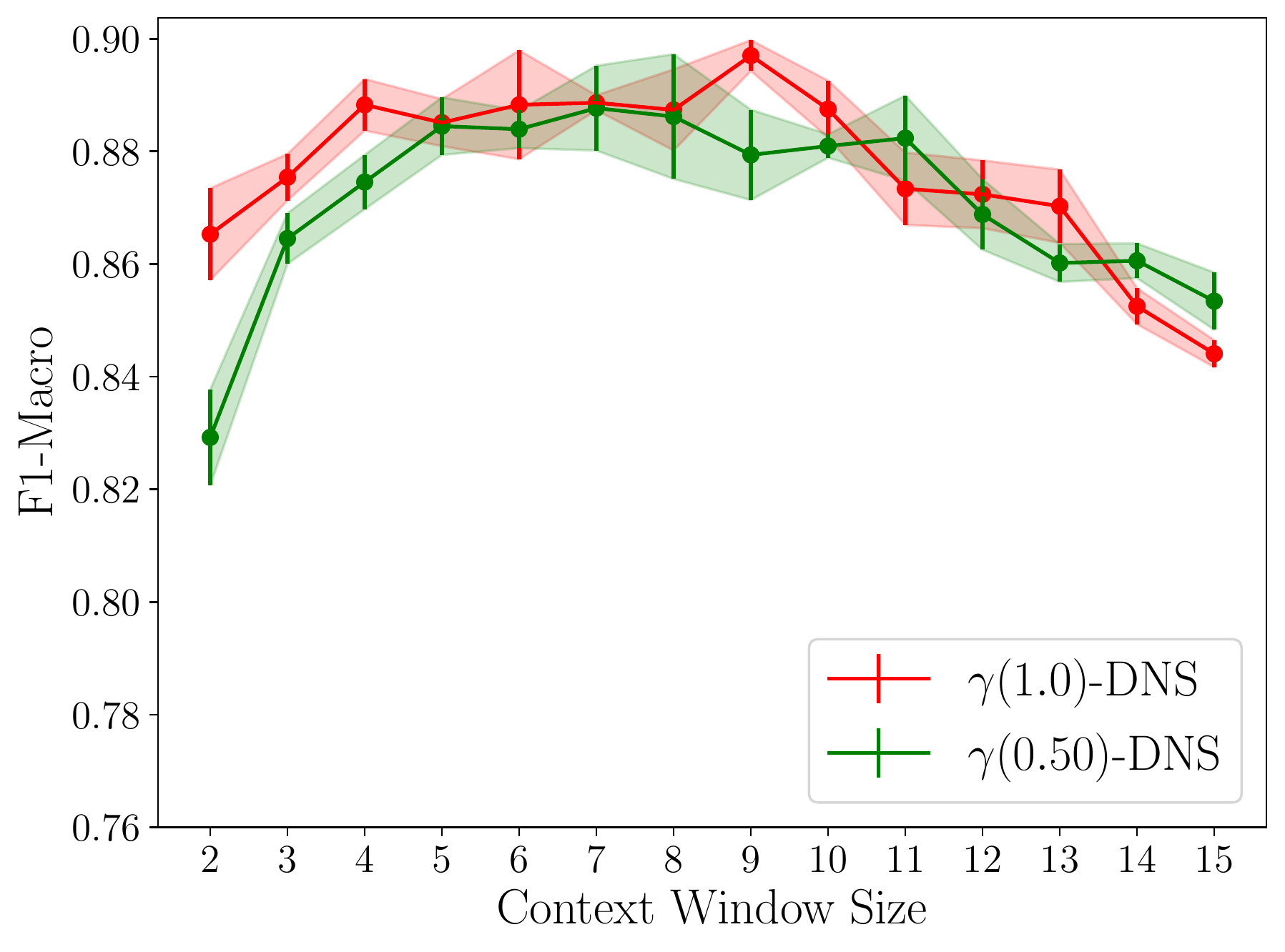}
\end{subfigure}
\begin{subfigure}{.33\textwidth}
  \centering
  \includegraphics[width=.99\linewidth]{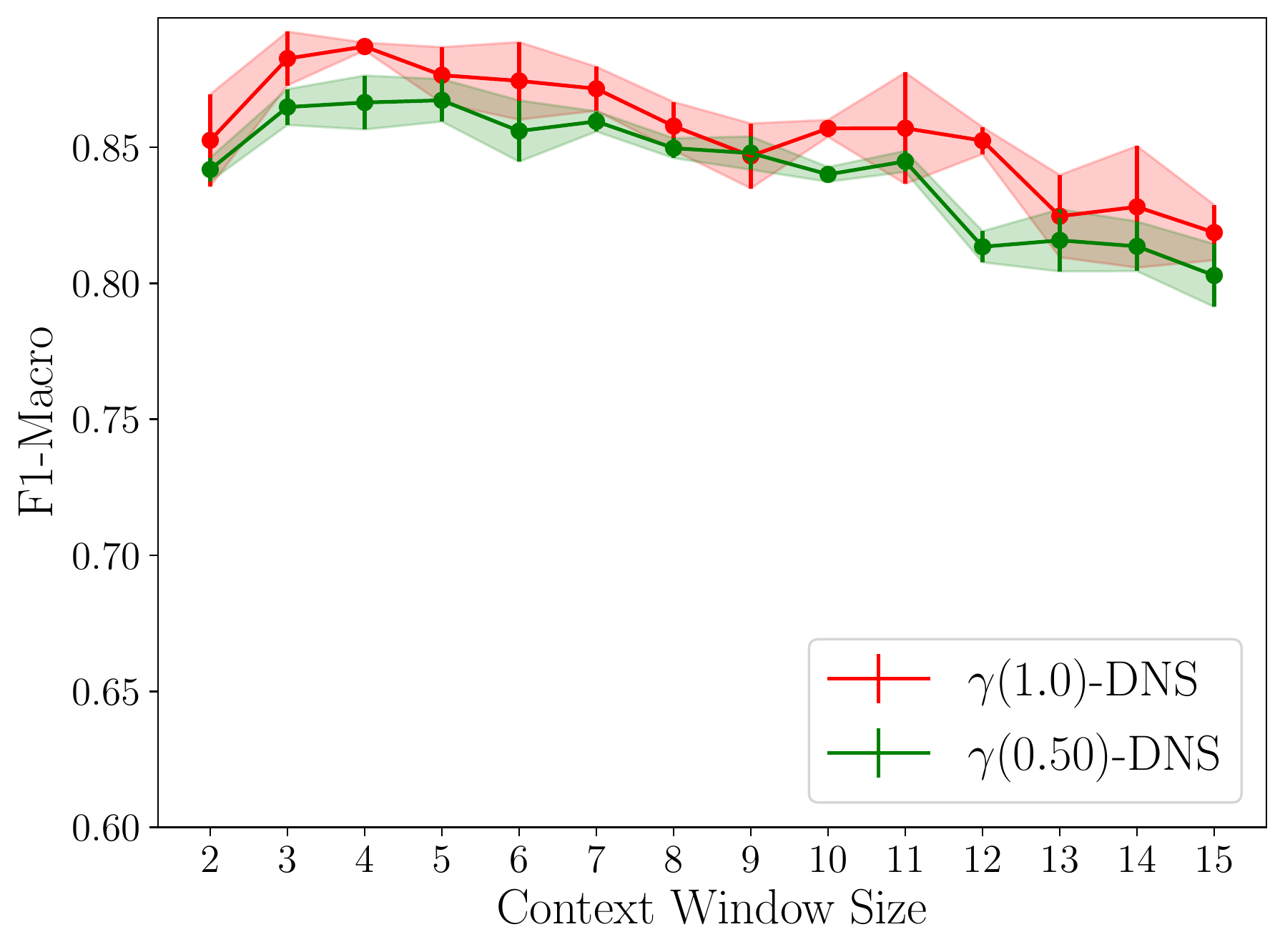}
\end{subfigure}
\begin{subfigure}{.33\textwidth}
  \centering
  \includegraphics[width=.99\linewidth]{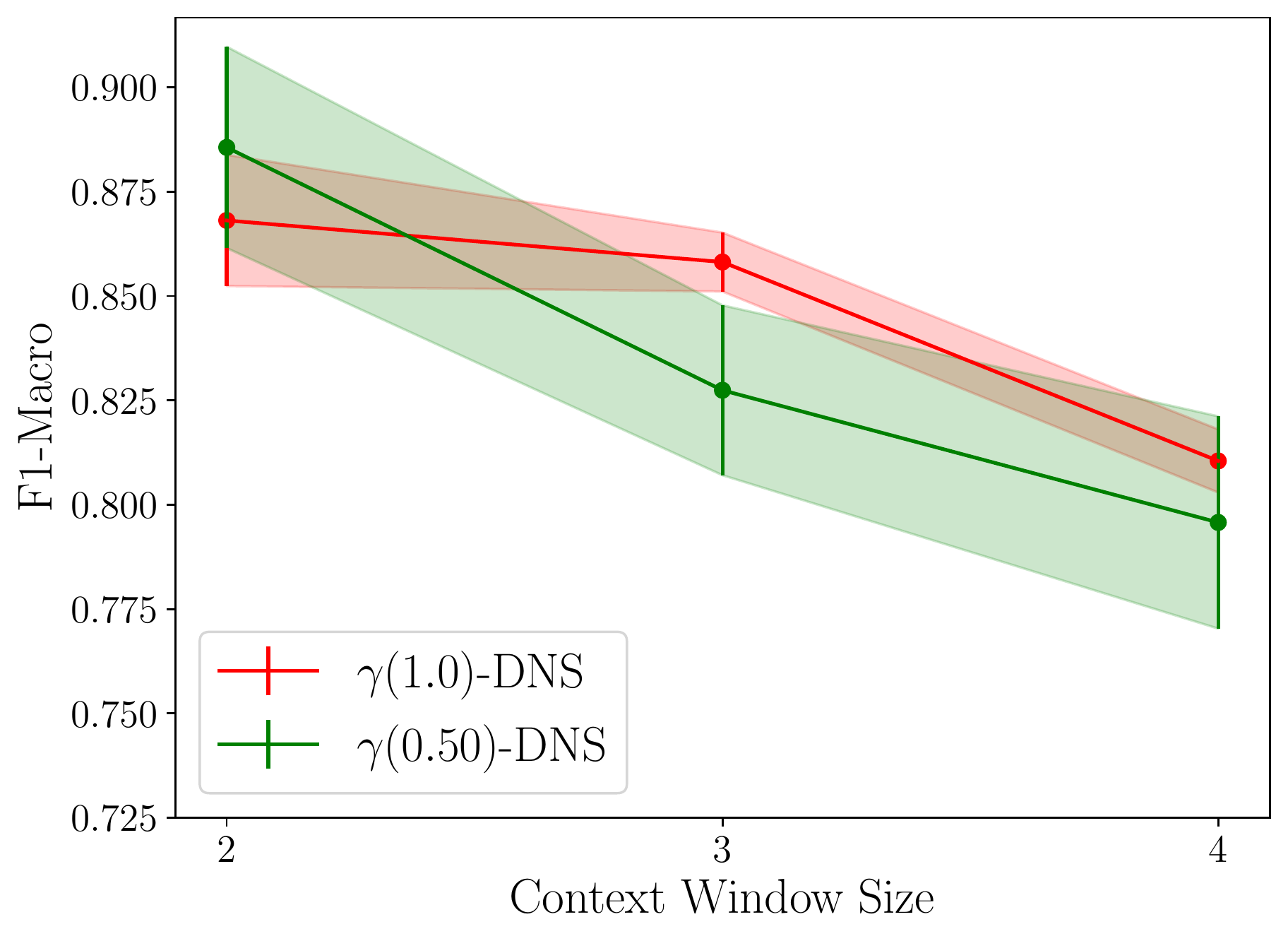}
\end{subfigure}

\bigskip
\begin{subfigure}{.33\textwidth}
  \centering
  \includegraphics[width=.99\linewidth]{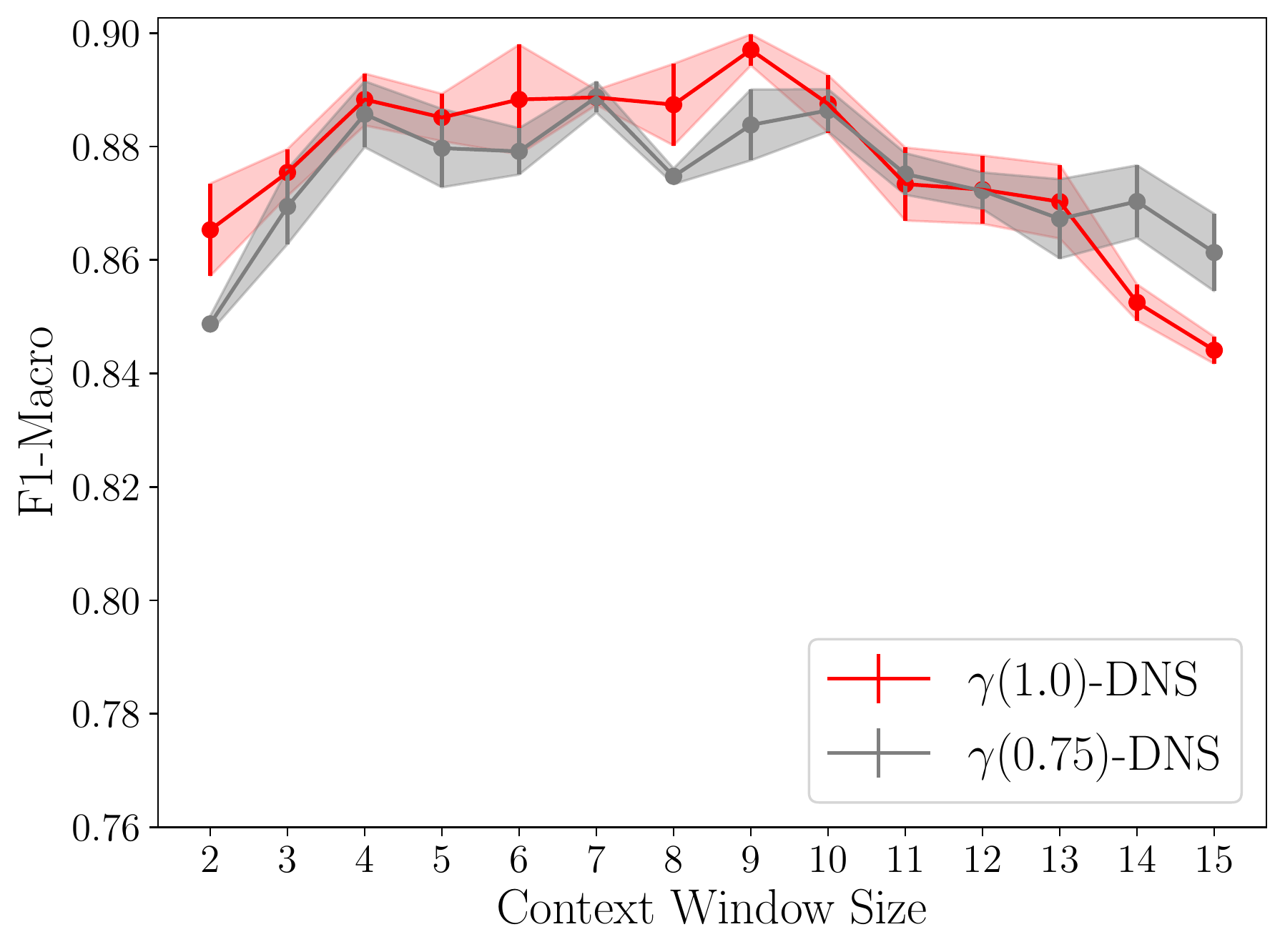}
\end{subfigure}
\begin{subfigure}{.33\textwidth}
  \centering
  \includegraphics[width=.99\linewidth]{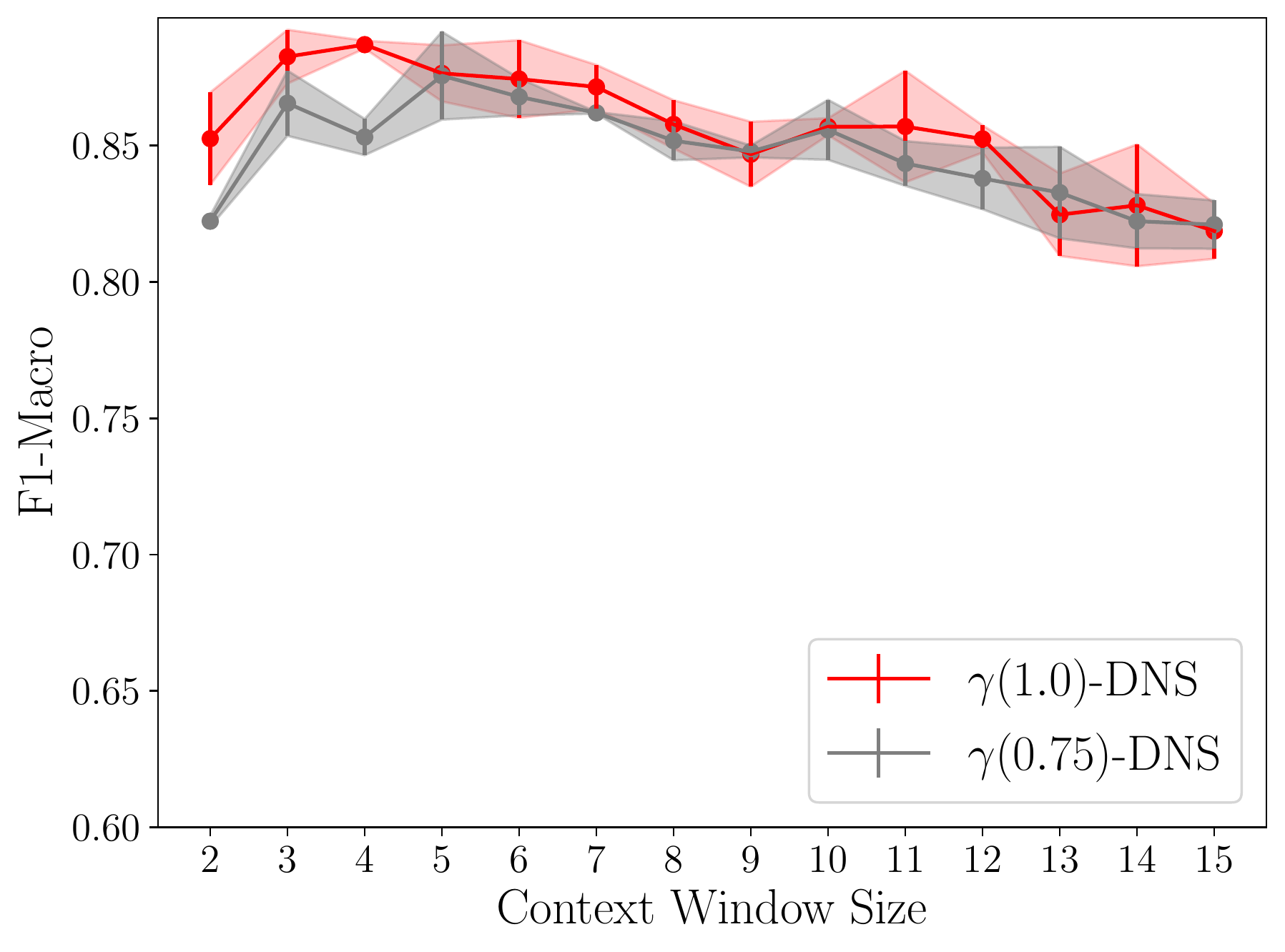}
\end{subfigure}
\begin{subfigure}{.33\textwidth}
  \centering
  \includegraphics[width=.99\linewidth]{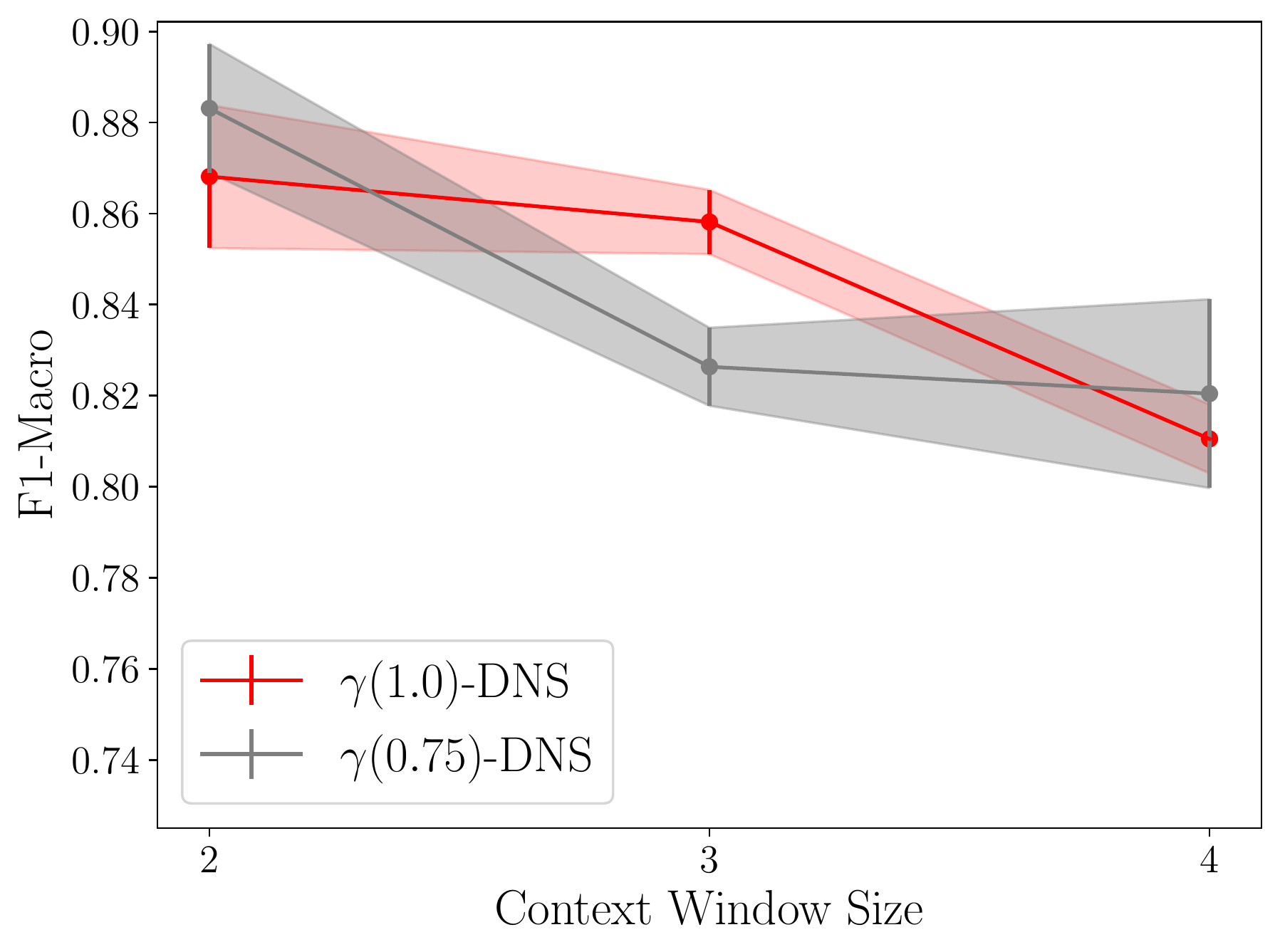}
\end{subfigure}

\bigskip
\begin{subfigure}{.33\textwidth}
  \centering
  \includegraphics[width=.99\linewidth]{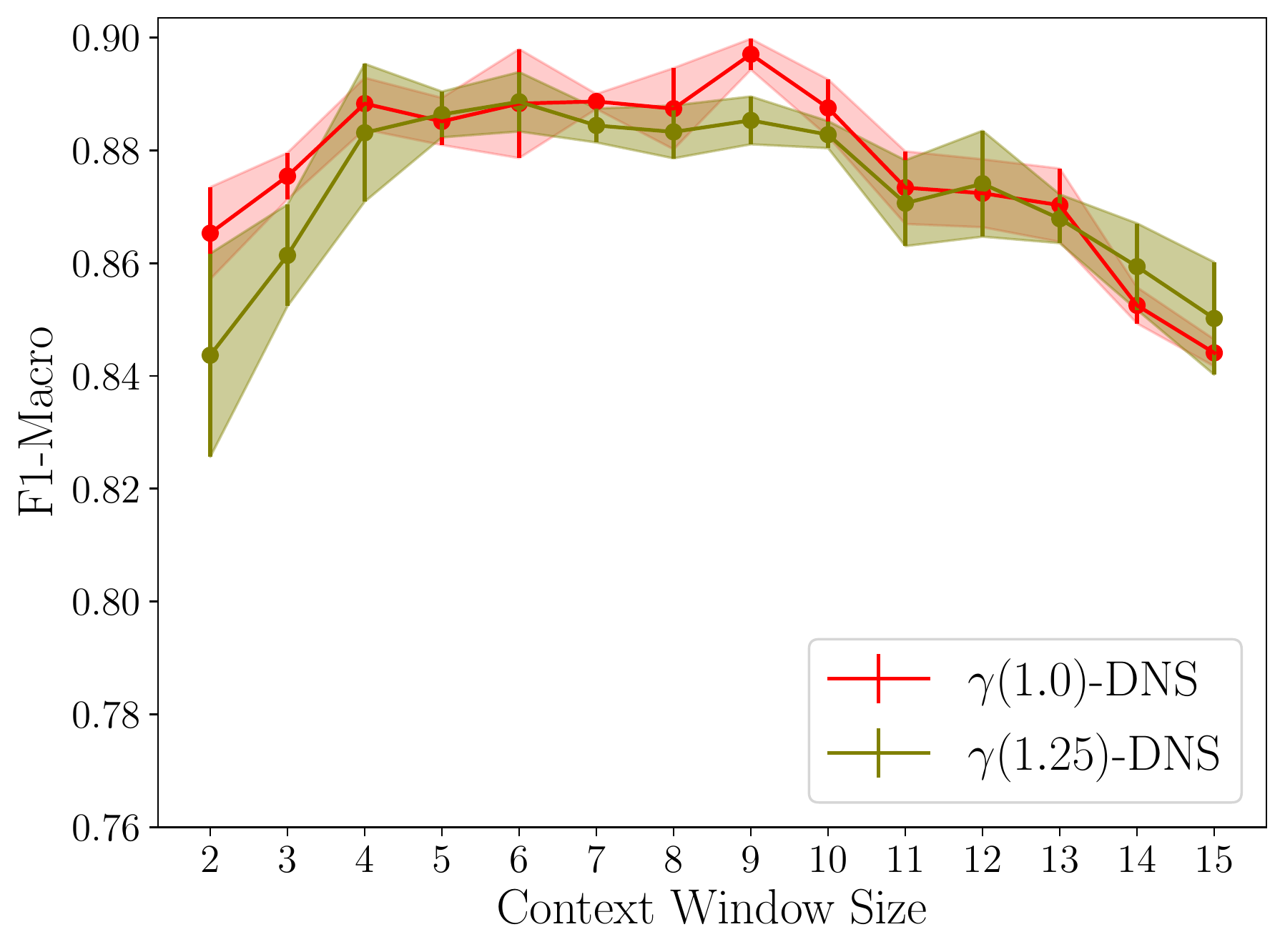}
  \caption*{Synthetic Sparse}
\end{subfigure}
\begin{subfigure}{.33\textwidth}
  \centering
  \includegraphics[width=.99\linewidth]{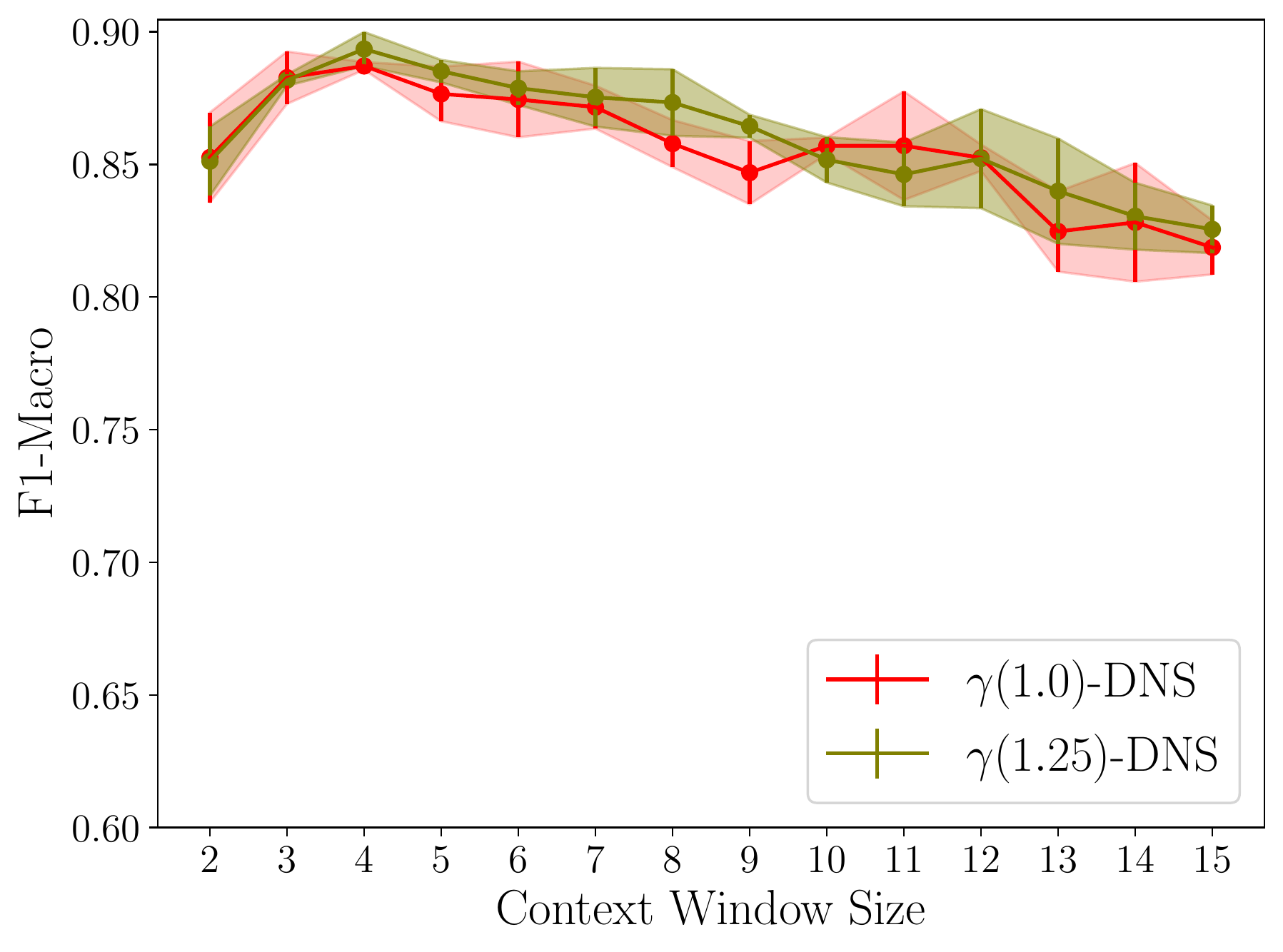}
  \caption*{Synthetic Moderate}
\end{subfigure}
\begin{subfigure}{.33\textwidth}
  \centering
  \includegraphics[width=.99\linewidth]{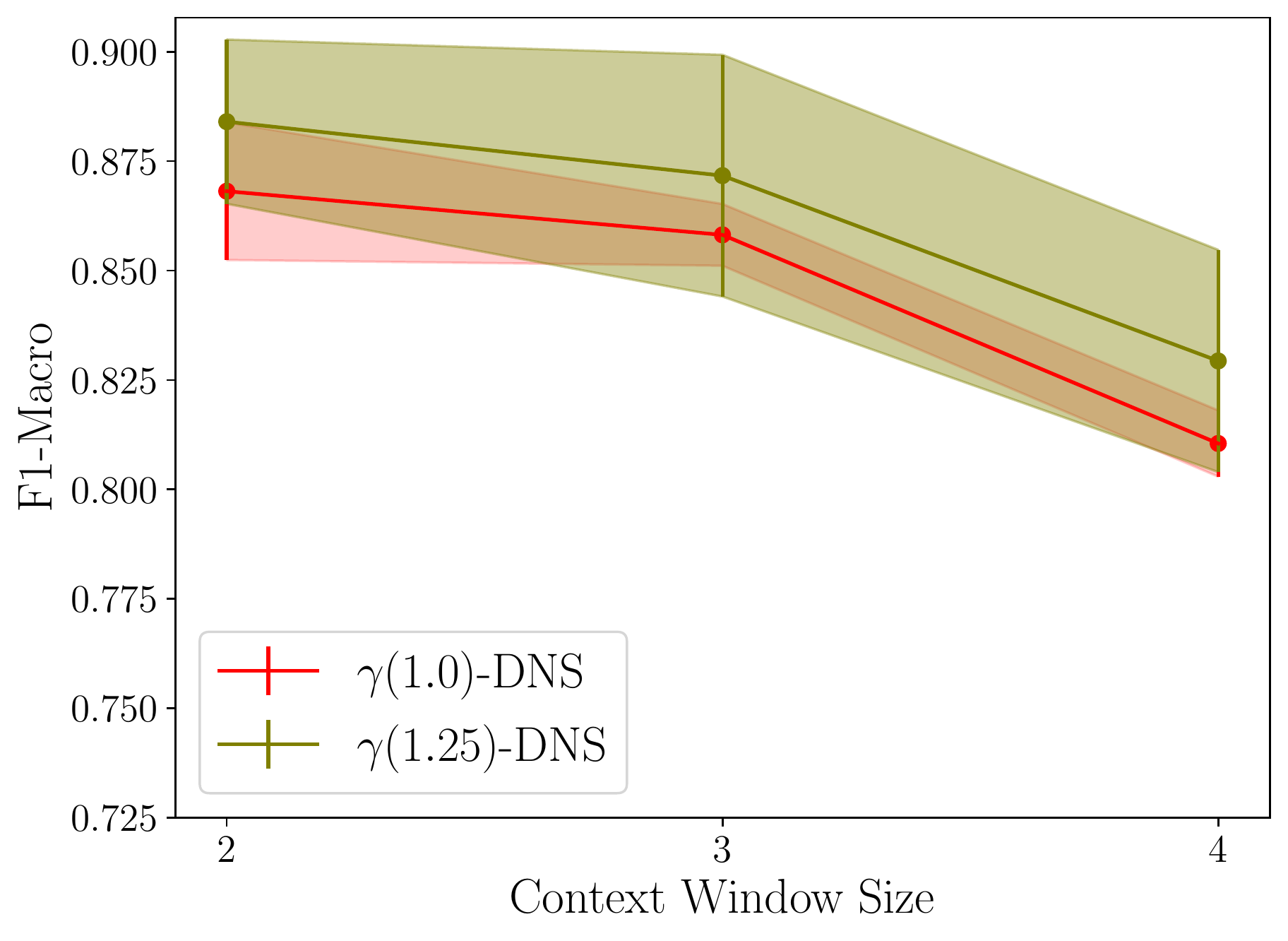}
  \caption*{Synthetic Dense}
\end{subfigure}
\caption{ Node classification performance (F1-Macro score) comparison for various $\gamma$-linear sampler based models with varying context window on Synthetic Sparse, Synthetic Moderate, and Synthetic Dense dataset.}
\label{fig:g-f1}
\end{figure*}

\subsubsection*{$\gamma$-linear Negative Sampling:}
We perform an empirical study to visualize the effect of varying $\gamma$ in $\gamma$-linear negative sampler on the synthetic datasets. We train the DeepWalk model with $\gamma$-linear negative sampler and denote it by $\gamma$-DNS. Moreover, we choose different values for $\gamma$ from 0 to 1.25 for this experiment (the models are denoted by $\gamma(value)$-DNS). In this experiment, we denote the DeepWalk-DNS model by $\gamma(1.0)$-DNS. From Figure \ref{fig:g-f1}, we see that $\gamma$ value closer to 1 follows the trend of the DeepWalk-DNS performance, whereas, $\gamma$ value closer to 0 follows the trend of the DeepWalk-UNS model performance. Theoretically, the performance of the $\gamma(0)$-DNS based model should be close to the UNS based model, but there is deviation across runs that require further investigation.

\subsection*{Sensitivity towards outlier points:}
To evaluate the sensitivity towards outlier points, we artificially add distant nodes with similar class values on the CiteSeer dataset. Figure \ref{fig:outlier} shows the sensitivity of our DNS model and the UNS model for these outlier points. Both models perform poorly with increasing number of outliers, which shows outliers hurt all models with the local similarity assumption.

\end{document}